\documentclass[11pt]{article}

\usepackage[T1]{fontenc}
\ifdefined\NEXAISSubmissionMode
  \usepackage[submission]{nexais}
\else
  \usepackage[preprint]{nexais}
\fi

\usepackage{newtxtext,newtxmath}
\usepackage[a4paper,margin=26mm]{geometry}
\usepackage[protrusion=true,expansion=false]{microtype}

\usepackage[round,authoryear]{natbib}
\setcitestyle{authoryear,round,citesep={;},aysep={,},yysep={;}}
\usepackage{booktabs}
\usepackage[flushleft]{threeparttable}
\usepackage[font=small,labelfont=bf]{caption}
\usepackage[shortlabels]{enumitem}
\usepackage{flafter}
\usepackage{float}
\usepackage{mathtools}
\usepackage{graphicx}
\usepackage[breaklinks=true]{hyperref}
\nexaishypersetup
\usepackage[capitalize,nameinlink,noabbrev]{cleveref}

\setlist{nosep,leftmargin=1.6em}
\providecommand{\off}{\operatorname{off}}
\providecommand{\Cov}{\operatorname{Cov}}
\providecommand{\var}{\operatorname{var}}
\providecommand{\rank}{\operatorname{rank}}
\providecommand{\inner}[2]{\left\langle#1,#2\right\rangle}

\providecommand{\cP}{\mathcal P}
\providecommand{\cR}{\mathcal R}

\providecommand{\bS}{\mathbf S}
\providecommand{\bT}{\mathbf T}

\providecommand{\bSigma}{\mathbf\Sigma}

\providecommand{\bx}{\mathbf x}

\providecommand{\ones}{\mathbf 1}
\providecommand{\Rex}{\mathcal E}
\providecommand{\op}{\mathrm{op}}
\providecommand{\F}{\mathrm F}
\providecommand{\Had}{\circ}
\newcommand{\Asig}{\mathsf{A}_{\bSigma}}
\newcommand{\Bsig}{\mathsf{B}_{\bSigma}}
\newcommand{\Cnu}{\mathsf{C}_{\nu}}
\newcommand{\Vnu}{\mathsf{V}_{\nu}}
\newcommand{\Gnu}{\mathsf{G}_{\nu}}

\newcommand{\paperabstracttext}{%

{Principal component analysis (PCA) can rotate away from its population
target when a covariance matrix is estimated from limited data.  We introduce
\emph{diagonal attenuation}, which preserves sample cross-covariances while
reducing coordinatewise sample variances.  The method is revealed exactly by
averaging a linear full-output reconstruction loss over random input masks;
studying the correction directly extends it beyond the range attainable by
masking.  We isolate the part of the random coupling between retained and
omitted population directions that is contributed by sample-variance errors,
and show how attenuation can reduce the resulting rotation.  Under balanced
marginal variances, we derive an explicit expected-risk theorem, uniform over
the attenuation path for all sufficiently large finite samples, and obtain the
asymptotically risk-minimizing strength.  For general covariances, we
characterize when attenuation leaves the population PCA subspace unchanged
and give a risk theorem that also accounts for changing eigengaps and the
population cost when the target moves.  Simulations track this tradeoff from
exact preservation back to PCA. {Across local image patches, speech
spectra, and smartphone acceleration, both mask-derived and direct attenuation improve PCA under two fitting-sample budgets, and one of them has the largest mean gain among seven methods in every data--budget cell.  The full path selects strengths beyond the mask-derived boundary on {$63\%$--$95\%$} of the
subsamples.}}}

\begin{document}

\title{Diagonal Attenuation: A Finite-Sample Correction for PCA}
\author{Qiang Sun\thanks{E-mail: qsunstats@gmail.com.}}
\affiliation{University of Toronto and MBZUAI}

\paperabstract{\paperabstracttext}

\date{September 4, 2026}
\pdfsubject{Diagonal attenuation for correcting finite-sample rotation in principal component analysis}
\pdfkeywords{diagonal attenuation, PCA, input masking, finite-sample rotation, covariance shrinkage}

\maketitle

\tableofcontents

\section{Introduction}\label{sec:introduction}

{
  Principal component analysis (PCA) estimates a low-dimensional
subspace by diagonalizing a sample covariance matrix.  With limited data,
sampling fluctuations couple population directions on opposite sides of the
rank cutoff and rotate the fitted subspace away from its population target;
the resulting error is controlled by both these random couplings and the
spectral gaps that resist them
\citep{anderson1963asymptotic,koltchinskii2017concentration,
reiss2020nonasymptotic,elhanchi2025geometric}. {A less visible source of
rotation is error in the coordinatewise sample variances.  These errors form a
diagonal matrix in the observed coordinates, but the same matrix generally has
off-diagonal entries in the population principal-component basis.  Those
entries couple a direction retained by PCA with one it omits and can therefore
rotate the fitted subspace.  This identifies a specific correction: can
weakening the sample diagonal reduce PCA's finite-sample reconstruction risk
while retaining its informative sample cross-covariances?}
}

{Input masking provides an unexpected route to precisely this
correction.  Let $\bx_1,\ldots,\bx_n\in\RR^p$ be centered observations,
$\cP_d$ the set of rank-$d$ orthogonal projectors, and $W$ a random diagonal
mask.  Consider reconstructing each complete observation from its masked
version:}
\begin{equation}\label{eq:mae-objective}
  \widehat P^{\rm mask}
  \in\argmin_{P\in\cP_d}
  \frac1n\sum_{i=1}^n
  \E_W\norm{\bx_i-PW\bx_i}_2^2.
\end{equation}
{Write $\bS_n=n^{-1}\sum_i\bx_i\bx_i^\trans$ for the sample
covariance and let $\diag(\bS_n)$ retain its diagonal.  Averaging exactly over
independent Bernoulli masks, as in marginalized denoising
\citep{chen2012marginalized}, turns \eqref{eq:mae-objective} into PCA applied
to}
\begin{equation}\label{eq:intro-attenuation}
  \bT_{n,\alpha}=\bS_n-\alpha\diag(\bS_n),
  \qquad \alpha=\frac{m}{1+m},
\end{equation}
{where $m$ is the masking rate.  Thus masking reveals a covariance
correction that retains every sample cross-covariance and multiplies every
sample variance by $1-\alpha$.  We call the resulting operation
\emph{diagonal attenuation}.  \Cref{fig:mechanism} summarizes its derivation
and the estimator family studied here.}

\begin{figure}[t!]
  \centering
  \includegraphics[width=0.98\linewidth]{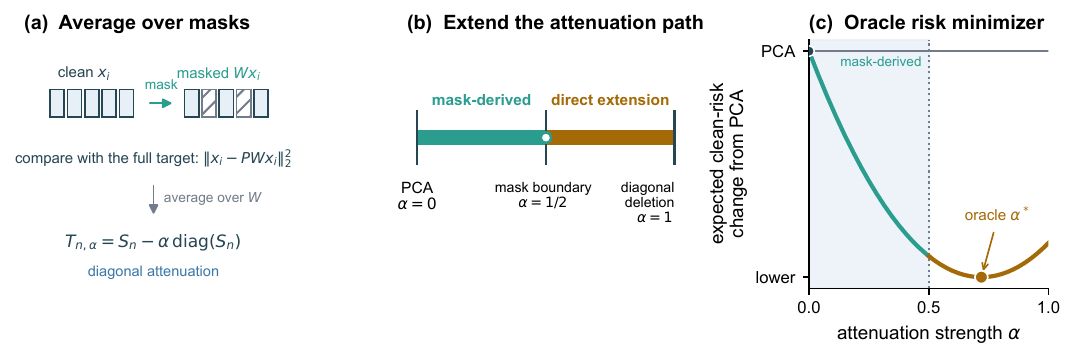}
  \caption{{Method overview.  \textbf{Left:} averaging the
  full-output reconstruction loss {in \eqref{eq:mae-objective}} over random input masks yields
  \eqref{eq:intro-attenuation}.  \textbf{Center:} input masking produces the
  first part of the attenuation path, while direct attenuation extends the
same correction from PCA to diagonal deletion.  \textbf{Right:} the oracle
strength balances the reduction in sampling-induced rotation against any
population cost caused by moving the target, and may lie beyond the
mask-derived range.}}
  \label{fig:mechanism}
\end{figure}

{Input masking reaches only $0\leq\alpha<1/2$.  We therefore study the
full path $0\leq\alpha\leq1$ directly: ordinary PCA is the endpoint
$\alpha=0$, while $\alpha=1$ gives diagonal-deletion, or hollowed, PCA
\citep{cai2021subspace,abbe2022lp}.  The strength can be selected by
reconstruction validation with PCA included as a candidate.}  {The statistical tradeoff now becomes explicit.  Attenuation reduces the part of the random retained--omitted coupling carried by coordinatewise variance fluctuations, but it can also alter the population eigenspace or its spectral separation.  We call the random cross-subspace component \emph{sampling leakage}.  Balanced marginal variances provide a transparent case: attenuation leaves the population target and all eigengaps unchanged, so its entire effect comes from changing this sampling leakage.  General covariances require two additional pieces---the attenuated eigengaps and the population cost if the target moves.  Our theory combines these quantities in a single comparison of clean reconstruction risk.}

{The paper makes three contributions.  First, we derive diagonal
attenuation exactly from masked-input reconstruction, extend the mask-derived range to a directly tunable PCA estimator, and characterize when the population PCA subspace is preserved or moved.  Second, we isolate the sample-diagonal contribution to the standard cross-subspace sampling
fluctuation and show how attenuation changes expected reconstruction risk.
Under balanced marginal variances, this yields an explicit theorem and the
asymptotically risk-minimizing strength; for general covariances, our risk
theorems account for changing eigengaps and population bias.  Third,
simulations verify the predicted risk tradeoff across preserved and moving
targets, while fitting-sample-budget studies on local image patches, speech
spectra, and smartphone acceleration compare seven methods under common
budgets.  Both attenuation paths improve PCA in every cell, and one has the
largest mean gain in all six cells.}

\subsection{Related work}\label{sec:related}

\paragraph{PCA and covariance regularization.}
{Classical and modern PCA theory describes eigenspace fluctuations,
perturbation bounds, and reconstruction risk
\citep{anderson1963asymptotic,koltchinskii2017concentration,
reiss2020nonasymptotic,elhanchi2025geometric}.  Linear shrinkage toward a
scalar identity improves covariance conditioning but preserves the sample
eigenvectors \citep{ledoit2004wellconditioned}; shrinkage toward a diagonal
target also modifies sample cross-covariances
\citep{schaefer2005shrinkage}.  Diagonal attenuation takes a different
direction: it retains the observed cross-covariances and weakens only the
sample diagonal.  Diagonal deletion is established for noisy or incomplete
low-rank matrices, and hollowed PCA can improve spectral recovery under
heteroskedastic noise \citep{cai2021subspace,abbe2022lp}.  HeteroPCA further
alternates diagonal deletion and imputation in a latent low-rank measurement
model \citep{zhang2022heteroskedastic}.  Our observations are instead complete
draws from the covariance whose PCA subspace is the target.  We study the
continuous path from no deletion to complete deletion and ask when a partial
correction lowers the clean reconstruction risk of ordinary PCA.}

\paragraph{Masking and linear reconstruction objectives.}
{Denoising autoencoders learn from corrupted inputs
\citep{vincent2008extracting,vincent2010stacked}, and marginalized denoising
integrates over the corruption to obtain a deterministic moment objective
\citep{chen2012marginalized,chen2014marginalized}.  An unmasked
undercomplete linear autoencoder recovers the PCA subspace under squared loss
\citep{baldi1989neural}.  Our orthogonal-projector formulation keeps both the
target and evaluation loss equal to those of PCA, allowing a signed clean-risk
comparison.  Canonical masked autoencoders instead score only masked targets
\citep{he2022masked}; \Cref{rem:masked-target} states this objective boundary,
and the supplement gives the corresponding linear characterization.}

\paragraph{Closest spectral analyses of masking.}
{\citet{ji2023power} connect complementary random masks in a contrastive
objective to diagonal deletion and feature recovery under a spiked additive
noise model.  Our full-output projector objective instead produces a
continuous partial-attenuation path and is evaluated against the ordinary PCA
subspace of the clean covariance.  \citet{bisulco2025linearity} study
full-output block masking with an
untied linear encoder and decoder.  For a fixed data matrix, they derive the
global minimizers through a generalized eigenvalue problem and characterize
all critical points.  We
instead constrain the reconstruction map to be an orthogonal projector and
compare its expected clean PCA risk under random sampling.
\citet{zurich2026masked} analyze coordinatewise masked self-supervised ridge
regression in proportional asymptotics; their estimator is an asymmetric
full-rank predictor trained to infer hidden coordinates.  Our estimator is a
fixed-rank projector, and our risk theory is fixed dimensional and centered on
the finite-sample PCA target.  More broadly, \citet{lin2024augmentation} show
that linear data augmentation induces data-dependent spectral
regularization.  Here we identify the particular cross-subspace fluctuations,
attenuated gaps, and population movement that determine the clean-risk
comparison with PCA.  Further objective-level distinctions are collected in
\cref{sec:supp-related}.}

\subsection{Notation}\label{sec:notation}

{For a symmetric matrix $A$, $\lambda_j(A)$ is its $j$th largest
eigenvalue, $\diag(A)$ retains its diagonal, and
$\off(A)=A-\diag(A)$.  The operator and Frobenius norms are
$\norm{A}_{\op}$ and $\norm{A}_{\F}$, and $\Had$ denotes the elementwise
product.  The set of rank-$d$ orthogonal projectors is $\cP_d$;
$P_A$ denotes a leading rank-$d$ projector of $A$, with a fixed rule for ties.
The mask rate is $m$, and $\alpha$ is the diagonal-attenuation strength.
Additional quantities are defined where they first appear.}

\section{Diagonal attenuation for PCA}\label{sec:model}

{This section defines the clean PCA target, derives diagonal attenuation from input masking, and extends it into a directly tunable estimator family. We then describe the population geometry needed to compare this family with PCA.}

\subsection{PCA target and clean risk}\label{sec:pca-target}

{Let $\bx_1,\ldots,\bx_n\in\RR^p$ be centered observations,
and define the sample covariance}
\begin{equation}\label{eq:sample-covariance}
  \bS_n=\frac1n\sum_{i=1}^n\bx_i\bx_i^\trans.
\end{equation}
Fix $1\leq d<p$.  A rank-$d$ PCA eigenspace is represented by an orthogonal projector
\begin{equation}\label{eq:symmetric-map}
  P=UU^\trans,
  \qquad U\in\RR^{p\times d},
  \qquad U^\trans U=I_d,
  \qquad P\in\cP_d,
\end{equation}
{where $\cP_d$ is the set of rank-$d$ orthogonal projectors.}
{For a fresh, unmasked observation $\bx$ with population covariance
$\bSigma=\E(\bx\bx^\trans)$, the reconstruction risk is}
\begin{equation}\label{eq:clean-risk}
  \cR(P)=\E\norm{\bx-P\bx}_2^2
  =\tr(\bSigma)-\tr(P\bSigma).
\end{equation}
Let $P_*$ be a leading rank-$d$ projector of $\bSigma$.  Its excess
reconstruction risk is
\begin{equation}\label{eq:excess-risk}
  \Rex(P)=\cR(P)-\cR(P_*)
  =\tr\{\bSigma(P_*-P)\}.
\end{equation}

\subsection{From input masking to diagonal attenuation}\label{sec:mask-identity}

{In \eqref{eq:mae-objective}, let
$W=\diag(w_1,\ldots,w_p)$, where the $w_j$ are independent
$\operatorname{Bernoulli}(q)$ variables.  Each coordinate is retained with
probability $q$, and $m=1-q$ is the mask rate.  Averaging over this mask
distribution gives the attenuation parameter and matrix}
\begin{equation}\label{eq:Talpha}
  \alpha=\frac{1-q}{2-q}=\frac{m}{1+m},
  \qquad
  \bT_{n,\alpha}=\bS_n-\alpha\cdot\diag(\bS_n).
\end{equation}
{Mask rates $0\leq m<1$ correspond exactly to
$0\leq\alpha<1/2$.}

\begin{proposition}[Exact mask marginalization]\label{prop:reduction}
    {Every optimizer of \eqref{eq:mae-objective} is a leading rank-$d$ spectral
    projector of $\bT_{n,\alpha}$.  If
    $\lambda_d(\bT_{n,\alpha})>\lambda_{d+1}(\bT_{n,\alpha})$, then}
\begin{equation}\label{eq:masked-spectral-estimator}
  \widehat P^{\rm mask}=P_{\bT_{n,\alpha}}.
\end{equation}
\end{proposition}

{\Cref{prop:reduction} shows that masking reveals a
deterministic covariance correction: it retains sample cross-covariances and
multiplies sample variances by $1-\alpha$.  Because masking reaches only
$0\leq\alpha<1/2$, we extend the same path directly and define}
\begin{equation}\label{eq:direct-family}
  \widehat P_\alpha=P_{\bT_{n,\alpha}},
  \qquad 0\leq\alpha\leq1,
\end{equation}
{where $P_A$ denotes a leading rank-$d$ projector of a
symmetric matrix $A$.  Ordinary PCA corresponds to $\alpha=0$, while
$\alpha=1$ is diagonal-deletion PCA based on $\off(\bS_n)$.  {Equivalently,
if $\bar s_n=\tr(\bS_n)/p$, then
$P_{\off(\bS_n)}=P_{\off(\bS_n)+\bar s_n I_p}$: for PCA, complete attenuation is equivalent to replacing all coordinate variances by
their common average while retaining all sample cross-covariances.}

The full interval
can be searched by held-out reconstruction with PCA included at the endpoint $\alpha=0$.}
{For a finite candidate grid, an independent-validation
implementation gives a formal fallback guarantee: without post-selection
refitting, the selected projector has clean risk no greater than the PCA
candidate plus an error that vanishes with the validation sample size
(\cref{prop:validation-fallback}).}

\begin{remark}[Relation to masked and linear autoencoders]
\label{rem:masked-target}
{Our full-output orthogonal-projector objective has the PCA subspace as
its unmasked target.  Canonical masked autoencoders score only masked targets,
and unconstrained linear autoencoders optimize a broader class of maps.  These
linear objectives are compared briefly in the supplement
\citep{baldi1989neural,he2022masked}.}
\end{remark}

\subsection{Population target under attenuation}\label{sec:geometry}

Write
\begin{equation}\label{eq:Tpopulation}
  D=\diag(\bSigma),
  \qquad
  \bT_\alpha=\bSigma-\alpha D,
  \qquad
  P_*^\perp=I_p-P_*.
\end{equation}
{Because $D$ may contain directional information, attenuation can move
the population PCA target.  We measure the resulting population reconstruction cost by}
\begin{equation}\label{eq:population-attenuated-projector}
  \mathsf{B}_{\rm pop}(\alpha)
  =\Rex(P_{\bT_\alpha})
  =\tr\{\bSigma(P_*-P_{\bT_\alpha})\}.
\end{equation}
{The nonnegative quantity $\mathsf{B}_{\rm pop}(\alpha)$ is the
reconstruction cost that remains even if $\bT_\alpha$ is known without sampling
error.  It is zero exactly when $P_{\bT_\alpha}$ is also optimal for the
population PCA risk.  Under the positive cutoff eigengap in
\cref{ass:target-gap}, this is equivalent to $P_{\bT_\alpha}=P_*$.  Because
$P_*^\perp\bSigma P_*=0$, the population cross-block is}
\begin{equation}\label{eq:population-cross-block}
 P_*^\perp\bT_\alpha P_*
 =-\alpha P_*^\perp D P_*.
\end{equation}
{This cross-block records whether attenuation mixes the target subspace
with its orthogonal complement.  It vanishes exactly when the range of $P_*$ is
invariant under $\bT_\alpha$.}

{For $\alpha\in[0,1]$, define the block gap as the difference
between the smallest population Rayleigh quotient retained by $P_*$ and the
largest one discarded by $P_*$:}
\begin{equation}\label{eq:preserved-gap}
  \delta_\alpha=
  \min_{{\norm{v}_2=1,\, P_*v=v}}
  v^\trans\bT_\alpha v
  -
  \max_{{\norm{v}_2=1,\, P_*v=0}}
  v^\trans\bT_\alpha v.
\end{equation}
{A positive $\delta_\alpha$ means that the attenuated population variance
retained by $P_*$ remains strictly above that discarded by $P_*$.}

\begin{proposition}[Zero target-shift characterization]
\label{prop:target-preservation}
Fix $\alpha>0$.  The projector $P_*$ is the unique leading rank-$d$
projector of $\bT_\alpha$ if and only if
\begin{equation}\label{eq:general-target-preservation}
  P_*^\perp D P_*=0,
  \qquad
  \delta_\alpha>0.
\end{equation}
Consequently, $\mathsf{B}_{\rm pop}(\alpha)=0$ under these conditions.
\end{proposition}

{The first condition prevents attenuation from mixing the target and
discarded subspaces.  The second prevents their spectral ordering from
reversing.  Together they are necessary and sufficient for $P_*$ to remain the
unique leading rank-$d$ population projector.  {In \cref{sec:theory}, we
begin with balanced marginals to make the mechanism explicit and then allow
general covariances with changing eigengaps and population target movement.
We next state the assumptions shared by these analyses.}}

\begin{assumption}[Gaussian sampling]\label{ass:gaussian-sampling}
{The observations satisfy
$\bx_1,\ldots,\bx_n\stackrel{\mathrm{iid}}{\sim}
\mathcal N(0,\bSigma)$ with $\bSigma\succ0$.}
\end{assumption}

\begin{assumption}[Fixed-dimensional identifiable target]\label{ass:target-gap}
{The dimension $p$ and target rank $d$ are fixed as $n$ grows,
and $\delta:=\lambda_d(\bSigma)-\lambda_{d+1}(\bSigma)>0$.}
\end{assumption}

\Cref{ass:gaussian-sampling,ass:target-gap} are standard in
fixed-dimensional PCA analyses
\citep{anderson1963asymptotic,koltchinskii2017concentration,
reiss2020nonasymptotic,yu2015useful}.
{Gaussianity gives transparent risk formulas; the supplement extends them
to distributions with finite moments.  The analysis treats the population mean
as known and zero, while in practice the observations are typically centered
empirically.  {The centered sample covariance has the same eigenspace
distribution as a zero-mean sample covariance based on $n-1$ observations, so
the same expected-risk results in \cref{sec:theory} apply with sample size
$n-1$.}}  Additional structure is introduced where the risk analysis requires.

\section{Why attenuation can improve finite-sample PCA}\label{sec:theory}

{For any symmetric matrix $M$, the cross-block $P_*^\perp MP_*$ records
how $M$ connects the population PCA subspace to its orthogonal complement.  We
call this block the \emph{cross-subspace leakage}; its individual entries are
couplings between retained and omitted directions.  For the attenuated sample
matrix $\bT_{n,\alpha}$, leakage separates into a deterministic population
component and a random sampling component.  Write
$E_n=\bS_n-\bSigma$.  Since $P_*^\perp\bSigma P_*=0$, diagonal attenuation gives
the exact decomposition}
\begin{equation}\label{eq:leakage-decomposition}
 \begin{aligned}
 P_*^\perp\bT_{n,\alpha}P_*
 & = \underbrace{-\alpha P_*^\perp DP_*}_{\text{population leakage}}
   +\underbrace{P_*^\perp\{E_n-\alpha\diag(E_n)\}P_*}
    _{\text{sampling leakage}}.
 \end{aligned}
\end{equation}
{The first term determines whether attenuation moves the population
target.  The second is mean-zero \emph{sampling leakage}; its
retained--omitted couplings drive sampling-induced rotation of the fitted
subspace.  We first
isolate this sampling effect under balanced marginal variances, where the
population term vanishes and all population eigengaps remain fixed.  {We then
allow unequal marginals, which can change the eigengaps and, when the target
moves, introduce population leakage.}}

\subsection{Balanced marginals: reducing sampling-induced rotation}
\label{sec:balanced-risk}

{Balanced marginal variances isolate the finite-sample effect of
attenuation.}
\begin{assumption}[Balanced marginal variances]\label{ass:balanced-marginals}
{For some $\tau>0$, $D=\tau I_p$.}
\end{assumption}

{The assumption equalizes the variances of the observed coordinates but
does not make the covariance isotropic: correlations, nontrivial principal
directions, and unequal eigenvalues remain allowed.  At the population level,}
\begin{equation}\label{eq:balanced-population-shift}
 \bT_\alpha=\bSigma-\alpha\tau I_p.
\end{equation}
{Therefore attenuation leaves the population principal subspace and every
eigengap unchanged.  {Any change in the estimated subspace must therefore come
from the sample fluctuation matrix $\bT_{n,\alpha}-\bT_\alpha$, specifically from
its couplings across the rank-$d$ cutoff.}  To see this, choose an orthonormal eigenbasis
$u_1,\ldots,u_p$ of $\bSigma$, ordered so that
$\bSigma u_r=\lambda_r u_r$, and fix $j\leq d<k$.  For one observation define}
\begin{align}\label{eq:single-observation-scores-main}
 a_{jk}(\bx)
 &=u_k^\trans(\bx\bx^\trans-\bSigma)u_j, \qquad
 b_{jk}(\bx)
  =u_k^\trans\diag(\bx\bx^\trans-\bSigma)u_j
   =\sum_{\ell=1}^p u_{k\ell}u_{j\ell}
    (x_\ell^2-\Sigma_{\ell\ell}),
\end{align}
{where $u_{r\ell}$ is the $\ell$-th coordinate of $u_r$.  The variable
$a_{jk}$ is one observation's contribution to the unattenuated coupling between
the retained direction $u_j$ and the omitted direction $u_k$.  The variable
$b_{jk}$ is the part of that contribution arising from coordinatewise variance
fluctuations.  Both are mean zero.  If $\overline a_{jk}$ and
$\overline b_{jk}$ denote their sample averages, then {the sampling-leakage
matrix collects all residual couplings across the rank cutoff:}}
\begin{equation}\label{eq:attenuated-cross-entry}
 \overline Z_{n,\alpha}
 :=P_*^\perp(\bT_{n,\alpha}-\bT_\alpha)P_*
 =\sum_{j\leq d<k}
  (\overline a_{jk}-\alpha\overline b_{jk})u_ku_j^\trans.
\end{equation}
{Its $(u_k,u_j)$ entry is
$\overline a_{jk}-\alpha\overline b_{jk}$.}  {At $\alpha=0$, PCA uses the {unattenuated coupling} $\overline a_{jk}$.
Attenuation subtracts its sample-diagonal component
$\alpha\overline b_{jk}$, leaving the residual coupling
$\overline a_{jk}-\alpha\overline b_{jk}$.  The next lemma shows how this
residual coupling produces subspace rotation and clean reconstruction risk.
{To express the resulting rotation, divide each residual coupling by its
population eigengap and define}}
\begin{equation}\label{eq:balanced-leading-rotation}
\begin{aligned}
 X_{n,\alpha}
 &:=\sum_{j\leq d<k}
   \frac{\overline a_{jk}-\alpha\overline b_{jk}}
        {\lambda_j-\lambda_k}u_ku_j^\trans,\qquad
 x_{n,kj,\alpha}:=u_k^\trans X_{n,\alpha}u_j
 =\frac{\overline a_{jk}-\alpha\overline b_{jk}}
        {\lambda_j-\lambda_k}. 
\end{aligned}
\end{equation}
\begin{lemma}[Residual coupling, rotation, and clean risk]
\label{lem:balanced-coupling-risk}
{{Under
\cref{ass:gaussian-sampling,ass:target-gap,ass:balanced-marginals}, for each
fixed $\alpha\in[0,1]$ there is a remainder $R_{n,\alpha}$ such that, for every
$j\leq d<k$,}}
\begin{equation}\label{eq:balanced-rotation-relation}
 \begin{aligned}
 P_*^\perp(\widehat P_\alpha-P_*)P_*
 &=X_{n,\alpha}+R_{n,\alpha}, ~~~~~~ \norm{R_{n,\alpha}}_{\F}=O_{\mathbb P}(n^{-1}),\\
 u_k^\trans(\widehat P_\alpha-P_*)u_j
 &=x_{n,kj,\alpha}+u_k^\trans R_{n,\alpha}u_j
 =x_{n,kj,\alpha}+O_{\mathbb P}(n^{-1}).
 \end{aligned}
\end{equation}
{Moreover,}
\begin{equation}\label{eq:balanced-local-risk}
 \Rex(\widehat P_\alpha)
 =\sum_{j\leq d<k}(\lambda_j-\lambda_k)x_{n,kj,\alpha}^2
  +O_{\mathbb P}(n^{-3/2})
 =\sum_{j\leq d<k}
   \frac{(\overline a_{jk}-\alpha\overline b_{jk})^2}
        {\lambda_j-\lambda_k}
   +O_{\mathbb P}(n^{-3/2}).
\end{equation}
\end{lemma}

{{The scalar $x_{n,kj,\alpha}$ is the first-order rotation from $u_j$
toward $u_k$, so a larger eigengap weakens the effect of the same residual
coupling.  The risk formula weights its square by that eigengap, leaving one
eigengap in the denominator.}  This is a full rank-$d$ statement and does not
require the pairwise couplings to be independent.  The proof is given in
\cref{sec:supp-balanced-pairs}.}
Although the individual {coordinates $x_{n,kj,\alpha}$} in
\eqref{eq:balanced-leading-rotation} depend on the chosen eigenbasis when
eigenvalues repeat within either block, their aggregate contribution in
\eqref{eq:balanced-local-risk} does not. {This basis invariance suggests a
coordinate-free risk analysis for general covariance, which we develop in
\cref{sec:general-fixed}.}

{It remains to determine whether attenuation reduces the variance of the
residual coupling.  Under Gaussian sampling, Isserlis' identity
\citep{isserlis1918formula} gives}
\begin{gather}
 \Cov\{a_{jk}(\bx),b_{jk}(\bx)\}
 =2\lambda_j\lambda_k\sum_{\ell=1}^p u_{j\ell}^2u_{k\ell}^2\geq0,~~~\text{and hence} \label{eq:coupling-diagonal-covariance}\\
 \var\{a_{jk}(\bx)-\alpha b_{jk}(\bx)\}-\var\{a_{jk}(\bx)\}
 =-2\alpha\Cov\{a_{jk}(\bx),b_{jk}(\bx)\}
  +\alpha^2\var\{b_{jk}(\bx)\}. \label{eq:balanced-pairwise-variance}
\end{gather}
{The nonnegative overlap term in
\eqref{eq:coupling-diagonal-covariance} quantifies how strongly the
sample-diagonal component aligns with the unattenuated coupling and thus why a
small positive attenuation reduces its variance; the quadratic term in
\eqref{eq:balanced-pairwise-variance} explains why excessive attenuation can
eventually increase it.}
Aggregating these contributions across the
rank cutoff, define
\begin{equation}\label{eq:balanced-risk-constant}
\begin{aligned}
 \mathcal K(\alpha)
 &=\sum_{j\leq d<k}
 \frac{\var\{a_{jk}(\bx)-\alpha b_{jk}(\bx)\}}
      {\lambda_j-\lambda_k},\\
 \mathsf C_{\bSigma}
 &=\sum_{j\leq d<k}
 \frac{\Cov\{a_{jk}(\bx),b_{jk}(\bx)\}}
      {\lambda_j-\lambda_k},
 \qquad
 \mathsf V_{\bSigma}
 =\sum_{j\leq d<k}
 \frac{\var\{b_{jk}(\bx)\}}
      {\lambda_j-\lambda_k}.
\end{aligned}
\end{equation}
{\Cref{thm:balanced-risk} aggregates these pairwise contributions and
controls the expected risk uniformly over the full attenuation path.}
\begin{theorem}[Expected risk under balanced marginals]
\label{thm:balanced-risk}
{Under \cref{ass:gaussian-sampling,ass:target-gap,ass:balanced-marginals},
there are finite constants $C_0$ and $N$, independent of $\alpha$, such that}
\begin{equation}\label{eq:balanced-uniform-remainder}
 \sup_{0\leq\alpha\leq1}
 \left|n\cdot\E\Rex(\widehat P_\alpha)-\mathcal K(\alpha)\right|
 \leq\frac{C_0}{\sqrt n},
 \qquad n\geq N.
\end{equation}
{Moreover, $\mathsf C_{\bSigma},\mathsf V_{\bSigma}>0$ and, for every
$0\leq\alpha\leq1$,}
\begin{equation}\label{eq:risk-difference}
 \mathcal K(\alpha)-\mathcal K(0)
 =-2\alpha\mathsf C_{\bSigma}+\alpha^2\mathsf V_{\bSigma}.
\end{equation}
\end{theorem}

Because \cref{thm:balanced-risk} shows that
$\mathcal K(\alpha)/n$ approximates the expected clean excess risk, we call
$\mathcal K(\alpha)$ the $n$-scaled theoretical risk curve.  Its minimizers
over the closures of the mask-derived and direct ranges are
\begin{equation}\label{eq:two-oracles}
 \alpha_{\rm mask,cl}^*
 =\min\!\left\{\frac{\mathsf C_{\bSigma}}{\mathsf V_{\bSigma}},\frac12\right\},
 \qquad
 \alpha_{\rm direct}^*
 =\min\!\left\{\frac{\mathsf C_{\bSigma}}{\mathsf V_{\bSigma}},1\right\}.
\end{equation}
{For any $n\geq N$, the theorem guarantees lower expected risk than PCA
whenever
$2\alpha\mathsf C_{\bSigma}-\alpha^2\mathsf V_{\bSigma}>2C_0/\sqrt n$.
Consequently, every fixed
$0<\alpha<\min\{2\mathsf C_{\bSigma}/\mathsf V_{\bSigma},1\}$ improves on PCA
for all sufficiently large finite samples.}  A larger mask rate need not be better,
and the optimum can lie beyond the mask-derived range.  The notation ``cl''
denotes closure: $\alpha=1/2$ is approached as the mask rate tends to one,
whereas direct attenuation includes the entire interval $[0,1]$; see
\cref{fig:mechanism}.

{As a concrete example, the rank-one flat-spike model
$\bSigma=\sigma^2I_p+\theta uu^\trans$ with $u=p^{-1/2}\ones$, {where $\ones$
is the all-ones vector,} has balanced
marginals.  Moreover, $\mathsf C_{\bSigma}/\mathsf V_{\bSigma}\geq1$, so
complete attenuation is optimal on the direct path.
\Cref{fig:target-shift-evidence}(a) verifies the theoretical risk curve in a
balanced multispike design.}

\subsection{General covariance: changing eigengaps and population leakage}\label{sec:general-fixed}

{When marginal variances are unequal, attenuation has two distinct
population-level effects.  It can change the eigengaps that resist
sampling-induced rotation even when the population PCA subspace is unchanged,
and it can create population leakage that moves this subspace.  We begin
with the target-preserving case, allowing the eigengaps to vary with
$\alpha$.  Recall the block gap $\delta_\alpha$ from
\eqref{eq:preserved-gap}.}

\begin{assumption}[Target preservation]\label{ass:uniform-target}
{For a compact attenuation interval $\mathcal I\subset[0,1]$ containing
zero,}
\begin{equation}\label{eq:uniform-target-preservation}
 P_*^\perp DP_*=0,
 \qquad
 \underline\delta_{\mathcal I}
 :=\inf_{\alpha\in\mathcal I}\delta_\alpha>0.
\end{equation}
\end{assumption}

{The first condition eliminates population leakage, so $\bT_\alpha$
leaves the range of $P_*$ invariant.  The positive block gap keeps this
subspace above its orthogonal complement in the spectral ordering.  Together,
the two conditions preserve the leading rank-$d$ population subspace over the
attenuation interval.  Neither requires equal coordinate variances:
\cref{fig:target-shift-evidence}(b) gives an unequal-marginal example satisfying
both.}

{{The general argument writes the three balanced-case steps in matrix
form: collect the sampling leakage, convert it into subspace rotation using the
attenuated eigengaps, and then convert that rotation into clean reconstruction
risk using the original PCA eigengaps.  For the first step, define the
one-observation leakage matrix $Z_\alpha(\bx)$; its sample average is the matrix
$\overline Z_{n,\alpha}$ introduced in \eqref{eq:attenuated-cross-entry}:}}
\begin{equation}\label{eq:sample-coupling}
 \begin{aligned}
 Z_\alpha(\bx)
 &=P_*^\perp\left[\bx\bx^\trans-\bSigma
 -\alpha\diag(\bx\bx^\trans-\bSigma)\right]P_*,\\
 \overline Z_{n,\alpha}
 &=\frac1n\sum_{i=1}^n Z_\alpha(\bx_i)
 =P_*^\perp(\bT_{n,\alpha}-\bT_\alpha)P_*.
 \end{aligned}
\end{equation}
{{The $(u_k,u_j)$ entry of $Z_\alpha(\bx)$ is
$a_{jk}(\bx)-\alpha b_{jk}(\bx)$, exactly the one-observation version of the
balanced residual coupling.  Under \cref{ass:uniform-target}, population
leakage vanishes, so $\overline Z_{n,\alpha}$ is the entire leakage across the rank cutoff.  The remaining difference is that attenuation may now
change the eigengaps.  To handle all rotation directions at once, let
$\mathbb X$ contain the matrices that map the target subspace into its
orthogonal complement:}}
\begin{equation}\label{eq:tangent-space}
 \mathbb X=\{X\in\RR^{p\times p}:X=P_*^\perp XP_*\}.
\end{equation}
{On this space, define the attenuated gap operator}
\begin{equation}\label{eq:rotation-operator}
 \mathcal L_\alpha(X)
 =X(P_*\bT_\alpha P_*)
 -(P_*^\perp\bT_\alpha P_*^\perp)X.
\end{equation}
{{The positive block gap makes $\mathcal L_\alpha$ invertible.  In
eigenbases of the two $\bT_\alpha$ blocks, this operator multiplies each
retained--omitted entry by its attenuated eigengap.  Its inverse therefore
converts the leakage matrix into the same leading rotation matrix used in the
balanced case:}}
\begin{equation}\label{eq:first-order-subspace-movement}
 X_{n,\alpha}=\mathcal L_\alpha^{-1}\overline Z_{n,\alpha}.
\end{equation}
{{Under balanced marginals, the attenuated eigengap is
$\lambda_j-\lambda_k$, so \eqref{eq:first-order-subspace-movement} reduces
exactly to \eqref{eq:balanced-leading-rotation}.  Its general
derivation is given in \cref{sec:supp-fixed-p}.}}

{{The final balanced-case step weights the squared rotation in direction
$(j,k)$ by the original eigengap $\lambda_j-\lambda_k$.  The following
operator applies the same clean-risk weighting to all directions at once:}}
\begin{equation}\label{eq:clean-risk-operator}
 \mathcal H_{\bSigma}(X)
 =X(P_*\bSigma P_*)
 -(P_*^\perp\bSigma P_*^\perp)X.
\end{equation}
{Indeed,
$u_k^\trans\mathcal H_{\bSigma}(X)u_j=(\lambda_j-\lambda_k)
u_k^\trans Xu_j$.  Thus, writing
$\langle A,B\rangle_{\F}=\tr(A^\trans B)$ for the Frobenius inner product,
the general theoretical risk curve is}
\begin{equation}\label{eq:general-risk-constant}
 \mathcal K(\alpha)
 =\E\left\langle
 \mathcal L_\alpha^{-1}Z_\alpha(\bx),
 \mathcal H_{\bSigma}
 \mathcal L_\alpha^{-1}Z_\alpha(\bx)
 \right\rangle_{\F}.
\end{equation}
{Equation \eqref{eq:general-risk-constant} is therefore the
coordinate-free counterpart of the pairwise risk in
\eqref{eq:balanced-local-risk}.}

\begin{theorem}[Expected risk under target preservation]
\label{thm:fixed-p}
{Under
\cref{ass:gaussian-sampling,ass:target-gap,ass:uniform-target}, there are
finite constants $C$ and $N$, independent of $\alpha$, such that}
\begin{equation}\label{eq:uniform-remainder}
 \sup_{\alpha\in\mathcal I}
 \left|n\E\Rex(\widehat P_\alpha)-\mathcal K(\alpha)\right|
 \leq\frac{C}{\sqrt n},
 \qquad n\geq N.
\end{equation}
\end{theorem}

{For any $n\geq N$, the theorem guarantees lower expected risk than PCA
whenever $\mathcal K(0)-\mathcal K(\alpha)>2C/\sqrt n$.}
{Unlike in the balanced case, this comparison now depends on both the random variation in the residual couplings
collected by $Z_\alpha$ and the attenuated eigengaps represented by
$\mathcal L_\alpha$.  When $D=\tau I_p$,
$\mathcal L_\alpha=\mathcal L_0=\mathcal H_{\bSigma}$, and
\eqref{eq:general-risk-constant} reduces exactly to the pairwise formula
\eqref{eq:balanced-risk-constant}.}

{Target preservation need not be exact for attenuation to improve
finite-sample PCA.  When $P_*^\perp DP_*\neq0$, attenuation creates population
leakage, and the leading population subspace of $\bT_\alpha$ moves away from
$P_*$.  Uniformly over a compact attenuation interval containing zero,
\cref{thm:target-moving-risk} in the supplement gives}
\begin{equation}\label{eq:target-moving-main}
 \E\Rex(\widehat P_\alpha)
 =\mathsf B_{\rm pop}(\alpha)
 +\frac{\mathcal K_{\rm move}(\alpha)}{n}
 +O(n^{-3/2}),
\end{equation}
{where $\mathsf B_{\rm pop}(\alpha)$ is the clean reconstruction cost of
moving the population subspace, as defined in
\eqref{eq:population-attenuated-projector}.  The term
$\mathcal K_{\rm move}(\alpha)/n$ is the signed finite-sample correction around
that moved subspace: sampling can rotate the estimator either toward or away
from $P_*$.  It reduces to
$\mathcal K(\alpha)/n$ when the target is preserved.  The theoretical risk
curve formed by the two displayed terms favors attenuation over PCA when}
\begin{equation}\label{eq:target-moving-benefit-main}
 n\mathsf B_{\rm pop}(\alpha)
 <\mathcal K(0)-\mathcal K_{\rm move}(\alpha).
\end{equation}
{Thus exact target preservation is unnecessary: attenuation helps
whenever the finite-sample gain on the right pays for the population cost on
the left.}

{In \cref{fig:target-shift-evidence}(c), increasing
population leakage moves the risk-minimizing strength from the direct-only
range through the mask-derived range and finally to $\alpha=0$, which recovers
PCA.  Across these regimes, the benefit is determined by the sampling leakage
remaining after attenuation, the eigengaps that translate this leakage into
rotation, and any population cost from moving the target.}

\section{Numerical studies}\label{sec:experiments}

\begin{figure}[t!]
  \centering
  \includegraphics[width=0.99\linewidth]
  {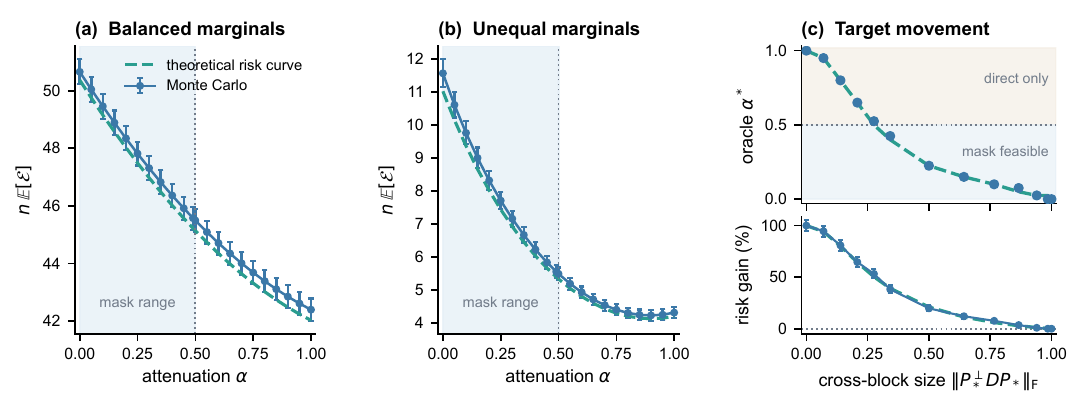}
  \caption{{Testing the risk theory from preserved to moving targets.
  \textbf{(a)} Balanced marginal variances and \textbf{(b)} unequal marginal
  variances whose diagonal preserves the PCA target.  Points and bars are
  Monte Carlo means and 95\% intervals; dashed curves are the theoretical risk
  curve from \eqref{eq:general-risk-constant}; the vertical scale is the
  $n$-scaled expected clean excess risk.  Blue shading marks the
  mask-derived range $0\leq\alpha<1/2$, with $\alpha=1/2$ shown as its limiting
  endpoint.  \textbf{(c)} Along \eqref{eq:moving-experiment-family}, the upper
  axis compares the strengths selected by the theoretical risk curve (line)
  and Monte Carlo mean risk (points), and the lower axis reports the
  corresponding risk reductions relative to PCA.  The horizontal coordinate
  is $\norm{P_{*,\phi}^\perp D_\phi P_{*,\phi}}_{\F}$, so moving right means
  greater population target movement.}}
  \label{fig:target-shift-evidence}
\end{figure}

This section presents numerical studies to examine the finite-sample performance of the proposed method.  Controlled
Gaussian experiments first ask whether the theoretical risk describes the
benefit of attenuation as the population target changes from preserved to
moving.  {Fixed fitting-sample-budget studies on three real-data modalities
then compare mask-derived and direct attenuation with PCA, correlation PCA,
HeteroPCA, and coordinate- and block-masked-target estimators under common
budgets.}  Gaussianity makes the signed risk
comparison explicit; the finite-moment theorem and non-Gaussian stress test in
the supplement show that the same sampling mechanism is not tied to Gaussian
observations.

\subsection{Finite-sample risk across preserved and moving targets}

\paragraph{Settings.}
{We begin with two target-preserving covariances.  The balanced design in
\cref{fig:target-shift-evidence}(a) has $p=16$, target rank $d=3$, and
eigenvectors given by normalized columns of a Hadamard matrix, so each loading
has magnitude $p^{-1/2}$. {Specifically, the covariance is
$I_p+\sum_{r=1}^3\theta_ru_ru_r^\trans$ with
$(\theta_1,\theta_2,\theta_3)=(8,4,2)$.}  Panel~(b) uses target rank two and the unequal
covariance}
\begin{equation}\label{eq:unequal-example}
 \bSigma=
 \begin{pmatrix}
 4&2&0.75&0.25\\
 2&4&0.25&0.75\\
 0.75&0.25&2.5&1.5\\
 0.25&0.75&1.5&2.5
 \end{pmatrix}.
\end{equation}
{Here $D=\diag(\bSigma)$ is not proportional to the identity, but it
leaves the leading two-dimensional target invariant and preserves its
spectral ordering along the attenuation path.}

{To move the population target continuously, panel~(c) uses the
bivariate family}
\begin{equation}\label{eq:moving-experiment-family}
 \bSigma_\phi
 =R_{\pi/4-\phi}
 \begin{pmatrix}5&0\\0&1\end{pmatrix}
 R_{\pi/4-\phi}^\trans,
 \qquad 0\leq\phi\leq\frac{\pi}{8}, \qquad
 R_\theta=
 \begin{pmatrix}
    \cos\theta&-\sin\theta\\
    \sin\theta&\cos\theta
 \end{pmatrix}.
\end{equation}
{Let $D_\phi=\diag(\bSigma_\phi)$ and let $P_{*,\phi}$ be its leading
population projector.  The population coupling introduced by attenuation is
measured by}
\begin{equation}
 \norm{P_{*,\phi}^\perp D_\phi P_{*,\phi}}_{\F}=\sin(4\phi).
\end{equation}
{It increases from zero to one along the path.  Thus panel~(c) changes
target movement while keeping the two population eigenvalues fixed; in
contrast, panel~(b) shows that unequal marginal variances alone need not move
the target.}

\paragraph{Evaluation.}
{For panels~(a) and (b), every point averages 3,000 Gaussian samples at
$n=256$ and reports $n\E\Rex(\widehat P_\alpha)$.  The theoretical curve is
$\mathcal K(\alpha)$ from \eqref{eq:general-risk-constant}, with no fitted
parameters.  For panel~(c), we use 13 values of $\phi$ from $0^\circ$ to
$22.5^\circ$, spaced by $1^\circ$ through $5^\circ$ and by $2.5^\circ$
thereafter.  At each value, 3,000 independent samples of size $n=64$ estimate
the mean clean risk over
$\{0,0.025,\ldots,1\}\cup\{0.49\}$.  We compare its minimizing strength with
the minimizer of
$\mathsf B_{\rm pop}(\alpha)+\mathcal K_{\rm move}(\alpha)/n$ on the same
grid.}

\paragraph{Results.}
The theoretical risk curve follows the Monte Carlo path throughout the
attenuation interval in both target-preserving designs.  The best simulated
attenuation lowers the scaled PCA risk by $16.3\%$ in the balanced design and
by $63.4\%$ at $\alpha=0.90$ in the unequal design.  Along the moving-target
path, the strength minimizing the theoretical risk differs from the Monte
Carlo mean-risk minimizer by at most one grid step, $0.025$, at all 13 movement
levels.  As the population coupling grows, the preferred strength decreases
from $\alpha=1$, passes through the mask-derived range, and eventually reaches
$\alpha=0$, recovering classical PCA.  The corresponding Monte Carlo risk reduction is
nearly $100\%$ under exact target preservation and remains $38.2\%$ when the
preferred strength first enters the mask-derived range, before tapering to
zero at the PCA endpoint.

\subsection{Fixed-budget studies on image patches and ordered signals}

\begin{figure}[!t]
  \centering
  \includegraphics[width=0.99\linewidth]
  {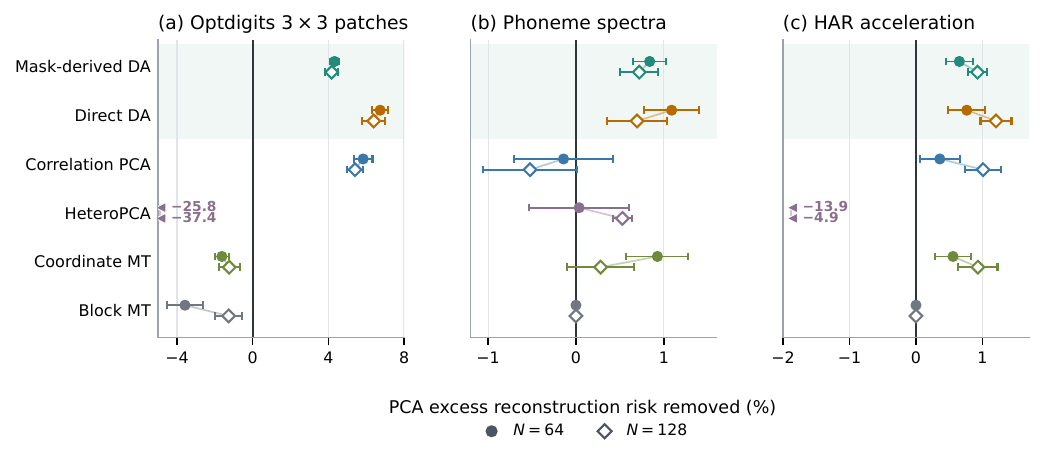}
  \caption{{Seven-method fixed fitting-sample-budget comparison across
  three real-data modalities.  Each row gives the mean percentage of PCA
  excess reconstruction risk removed by one alternative; positive values
  favor that method.  Filled circles and open diamonds correspond to
  $N=64$ and $128$, respectively, and the shaded rows are the two attenuation
  paths.  {Optdigits intervals use ten patch-location means, each
  averaging 100 training-subsample fits; Phoneme and HAR intervals use 100
  training subsamples.}  Within each fit, all methods use the same sampled observations,
  validation folds, and official test set.
  HeteroPCA means beyond the
  displayed range are marked at the left boundary and labeled by their
  values.}}
  \label{fig:fixed-budget-real-data}
\end{figure}

\paragraph{Settings.}
{A fitting-sample budget $N$ gives every method the same $N$
observations for fitting and five-fold reconstruction cross-validation.  The
selected method is then refit on all $N$ observations and evaluated on the
official test set.  We use $N\in\{64,128\}$ throughout.}

{The Optdigits data comprise 3,823 training and 1,797 test grayscale images of size $8\times8$ \citep{alpaydin1998optdigits}.   
Each image contributes one $3\times3$ patch: before examining its pixel values, we sample the patch's upper-left corner uniformly from the 36 possible locations and flatten the patch into a vector in $\RR^9$, {so $p=9$.}  
 Thus the nine coordinates record
relative positions within a local patch, while the sampled patches range over
  the full image.  We set $d=2$.  Ten independent patch-location seeds define
  ten versions of the training and test data; {for each seed and budget, we draw
  100 training subsamples.  Within each subsample, all methods receive the same
  training patches and the same five validation folds.}}

{The Phoneme data contain 256 ordered log-periodogram ordinates for each
speech frame \citep{hastie2009elements}.  We retain the official
speaker-disjoint split of 3,340 training and 1,169 test frames, set $d=8$,
{and use all $p=256$ coordinates.  For the structured masked-target
baseline, we group every 16 adjacent frequency bins.}
The UCI Human Activity Recognition data contain three-axis total acceleration
over 128 time points in each 2.56-second window
\citep{reyesortiz2013har}.  We use the magnitude series, retain the official
  split of 7,352 windows from 21 training subjects and 2,947 windows from 9 test subjects, {use all $p=128$ coordinates, set $d=16$, and group every
  eight adjacent time points for the structured masked-target baseline.}
  {Each ordered-signal cell uses 100 training subsamples.  All methods
  again share the observations and validation folds within each subsample,
  with all observations from the same speaker or subject kept in one fold.
  Digit, speech, and activity labels are never used.}}

\paragraph{Methods and evaluation.}
{Every data set compares ordinary PCA with six alternatives under the
same fitting budget.  Mask-derived and direct diagonal attenuation retain the
original coordinate scales and select $\alpha$ from
$\{0,0.05,\ldots,1\}\cup\{0.49\}$, restricted to $[0,0.49]$ for the
mask-derived path.  Correlation PCA instead standardizes coordinates before
estimating the subspace, while HeteroPCA iteratively imputes the sample
diagonal.  Coordinate and block masked-target projectors select a mask rate
  from $\{0,0.25,0.50,0.75,0.90\}$, which includes PCA at zero.  {The
  structured coordinate groups are the three rows of an Optdigits patch,
  consecutive sets of 16 frequency bins for Phoneme, and consecutive sets of
  eight time points for HAR.}  Definitions and
optimization details are given in \cref{sec:supp-real-data}.}

{For replicate $b$, let $\widehat R_{b,{\rm test}}$ be a method's test
reconstruction mean squared error and let $\widehat R_{b,{\rm test}}^{\rm PCA}$
be the test PCA error.  {We compute $R_{{\rm test}}^{(d)}$ by centering the
official test observations and summing the smallest $p-d$ eigenvalues of
their empirical covariance; equivalently, it is the reconstruction error of
the best rank-$d$ affine subspace fitted to the test sample.}  We report the
percentage of PCA's test excess reconstruction risk removed,}
\begin{equation}\label{eq:real-data-gain}
 G_b
 =100\,
 \frac{\widehat R_{b,{\rm test}}^{\rm PCA}-\widehat R_{b,{\rm test}}}
 {\widehat R_{b,{\rm test}}^{\rm PCA}-R_{{\rm test}}^{(d)}}.
\end{equation}
{The test reference in the denominator only expresses the final error
reduction on a common scale, which is not used for model selection.
}

\begin{table}[!htbp]
\centering
\scriptsize
\setlength{\tabcolsep}{4.1pt}
\caption{{Comparison under the same fitting-sample budgets.  Entries are the mean percentage of PCA test excess reconstruction risk removed, so PCA is zero and larger is better.  DA denotes diagonal attenuation and MT denotes masked-target training.  The best method in each data--budget cell is bold.}}
\label{tab:main-real-data-comparison}
{
\begin{tabular}{@{}lrrrrrr@{}}
\toprule
& \multicolumn{2}{c}{Optdigits patches} & \multicolumn{2}{c}{Phoneme} & \multicolumn{2}{c}{HAR} \\
\cmidrule(lr){2-3}\cmidrule(lr){4-5}\cmidrule(lr){6-7}
Method & $N=64$ & $N=128$ & $N=64$ & $N=128$ & $N=64$ & $N=128$ \\
\midrule
PCA & $0.00$ & $0.00$ & $0.00$ & $0.00$ & $0.00$ & $0.00$ \\
\addlinespace[1pt]
Mask-derived DA & $4.32$ & $4.18$ & $0.84$ & $\mathbf{0.72}$ & $0.65$ & $0.93$ \\
Direct DA & $\mathbf{6.74}$ & $\mathbf{6.40}$ & $\mathbf{1.09}$ & $0.70$ & $\mathbf{0.76}$ & $\mathbf{1.20}$ \\
\addlinespace[1pt]
Correlation PCA & $5.84$ & $5.41$ & $-0.14$ & $-0.52$ & $0.36$ & $1.01$ \\
HeteroPCA & $-25.79$ & $-37.36$ & $0.04$ & $0.53$ & $-13.94$ & $-4.87$ \\
Coordinate MT & $-1.63$ & $-1.24$ & $0.93$ & $0.28$ & $0.56$ & $0.93$ \\
Block MT & $-3.58$ & $-1.27$ & $0.00$ & $0.00$ & $0.00$ & $0.00$ \\
\bottomrule
\end{tabular}}
\end{table}

\paragraph{Results.}
{\Cref{fig:fixed-budget-real-data,tab:main-real-data-comparison} compare
PCA with all six alternatives.  Both attenuation paths have positive mean
gains with intervals above zero in all six data--budget cells.  The image
experiment produces the largest improvements.  At $N=64$ and $128$,
mask-derived attenuation removes
 {$4.32\%$ and $4.18\%$} of PCA's test excess reconstruction risk, while direct
 attenuation removes {$6.74\%$ and $6.40\%$}.  On Phoneme and HAR, the
corresponding ranges are $0.65\%$--$0.93\%$ for the mask-derived path and
$0.70\%$--$1.20\%$ for the direct path.  The direct path selects
 $\alpha>1/2$ on {$94\%$--$95\%$} of the image-patch fits and
$63\%$--$80\%$ of the ordered-signal fits, so the mask-derived boundary is
active in every data set.  \Cref{tab:main-real-data-comparison} places these
gains beside all four comparison methods.  One of the two attenuation paths
has the largest mean gain in every data--budget cell: direct attenuation leads
in five cells, and mask-derived attenuation leads for Phoneme at $N=128$.
Correlation PCA is competitive on the image patches and HAR, and coordinate
masked-target training helps on three of the four ordered-signal cells, but
  neither matches this consistency across modalities.  {Exact intervals,
  selected strengths, and optimization details are reported in
  \cref{sec:supp-real-data}.}}

\section{Discussion}\label{sec:discussion}

{Input masking reveals a useful but incomplete family of PCA estimators.
Diagonal attenuation extends that family over the full path from ordinary PCA
to hollowed PCA, while preserving every sample cross-covariance.  The theory
explains its benefit through sampling leakage: errors in the coordinatewise
sample variances become random retained--omitted couplings in the population
principal-component basis, and attenuation can weaken those couplings before
they rotate the fitted subspace.}
{This mechanism also identifies the tradeoff.  Balanced marginal
variances yield an explicit expected-risk theorem because attenuation changes
sampling leakage without moving the population target or its eigengaps.  For a
general covariance, the same finite-sample reduction must be compared with
the attenuated eigengaps and, when the target moves, its population
reconstruction cost.  The Gaussian experiments verify this combined risk
description from exact target preservation to the point at which PCA becomes
the preferred endpoint.}

{More broadly, marginalizing an artificial corruption can expose the
deterministic estimator that it induces.  Here that step turns input masking
into a covariance path whose statistical effect can be studied directly.  The
main risk theory is fixed dimensional and Gaussian; the three real-data
studies test the mechanism across local image, frequency, and temporal
coordinates.  The supplement extends the balanced formula under finite-moment conditions and
records the distinct masked-target and unconstrained linear objectives.
Structured attenuation beyond a common diagonal strength is a natural
direction for future work.}

\FloatBarrier
\section*{AI use statement}
Generative-AI tools assisted with literature discovery and synthesis, proof
cross-checking, code review, and language editing.  The author verified the
cited primary sources, derivations, and reported computations and takes full
responsibility for the manuscript.

\begingroup
\small
\setlength{\bibsep}{0pt}
\bibliographystyle{icml2026}
\bibliography{paper}
\endgroup

\beginsupplement*

\paragraph{Overview of the supplement.}
{The supplement follows the argument of the main text.
\Cref{sec:supp-mask} proves the exact mask marginalization that constructs the
estimator.  \Cref{sec:supp-population-geometry} develops its population
geometry, and \Cref{sec:supp-general-risk} gives the proofs of the
target-preserving, target-moving, and balanced-marginal risk results.
\Cref{sec:supp-extensions-practical} extends the balanced theorem beyond
Gaussian observations and establishes the validation fallback guarantee.
\Cref{sec:supp-objective-boundaries} compares the full-output projector with
masked-target and unconstrained linear reconstruction objectives.  Finally,
\cref{sec:supp-numerics} reports the non-Gaussian stress test, full real-data
tables, and implementation details for the empirical comparisons.}

\paragraph{Guide to the proofs of the main results.}
{The following list maps each main result to its proof, in main-text
order.}
\begin{enumerate}[label=(\roman*),leftmargin=*,itemsep=2pt,topsep=2pt]
\item {\Cref{prop:reduction} is proved in \Cref{sec:supp-mask} by
expanding the mask-averaged reconstruction loss and reducing its optimization
to the leading eigenspace of $\bT_{n,\alpha}$.}
\item {\Cref{prop:target-preservation} and the population-bias bound
\eqref{eq:population-bias-bound} are proved in
\cref{sec:supp-population-geometry}.  The first uses the
$P_*\oplus P_*^\perp$ block decomposition; the second integrates the
derivative of the population projector along the attenuation path.}
\item {\Cref{lem:balanced-coupling-risk} is proved in
\Cref{sec:supp-balanced-pairs}.  Spectral-projector commutation converts each
residual cross-entry into a rotation coordinate, and the common projector
expansion converts those coordinates into clean reconstruction risk.}
\item {\Cref{thm:balanced-risk} is proved in
\Cref{sec:supp-balanced-proof}.  Gaussian pairwise moments make the variance of
each retained-to-omitted residual random coupling a quadratic function of
$\alpha$; their eigengap-weighted sum gives the theoretical risk and its
optimal mask-derived and direct strengths.}
\item {\Cref{thm:fixed-p} is proved in \Cref{sec:supp-fixed-p}, using
the common clean-risk expansion in \cref{sec:supp-two-matrix}.  The proof
converts cross-subspace covariance noise into projector rotation and then into
expected clean reconstruction risk.}
\item {The target-moving expansion \eqref{eq:target-moving-main} is
formalized and proved in \Cref{sec:supp-target-moving-proof}.  It expands
around $P_{\bT_\alpha}$ and adds the population cost of moving away from
$P_*$.}
\end{enumerate}

\supplementtableofcontents

\section{Mask marginalization and estimator construction}
\label{sec:supp-mask}

{This section proves the exact reduction from masked-input
reconstruction to diagonal attenuation used in \Cref{prop:reduction}.}

\begin{proof}[Proof of \Cref{prop:reduction}]
For any symmetric projector $P$ and diagonal mask $W$,
\begin{equation}
 \norm{x-PWx}_2^2
 =x^\trans x-2x^\trans PWx+x^\trans WPWx
 =x^\trans x-
 \tr\!\left[P\{Wxx^\trans+xx^\trans W-Wxx^\trans W\}\right].
 \label{eq:supp-single-mask}
\end{equation}
The second equality uses the fact that the two scalar cross terms are equal.
For independent Bernoulli-$q$ mask coordinates,
\begin{align*}
 \E(Wxx^\trans)&=qxx^\trans, \\
 \E(Wxx^\trans W)&=q^2xx^\trans+q(1-q)\diag(xx^\trans).
\end{align*}
Taking the expectation in \eqref{eq:supp-single-mask} therefore produces
\begin{align*}
 2qxx^\trans-q^2xx^\trans-q(1-q)\diag(xx^\trans)
 &=q(2-q)\left\{xx^\trans-
 \frac{1-q}{2-q}\diag(xx^\trans)\right\}.
\end{align*}
The positive prefactor $q(2-q)$ does not affect the optimizer for $q>0$.
After averaging over observations, minimizing the loss is equivalent to
maximizing $\tr(P\bT_{n,\alpha})$ over $P\in\cP_d$.  Ky Fan's variational
principle \citep{bhatia1997matrix} gives a leading rank-$d$ projector of
$\bT_{n,\alpha}$, proving \Cref{prop:reduction} and the mapping
$\alpha=(1-q)/(2-q)=m/(1+m)$.
\end{proof}

\section{Population geometry under attenuation}
\label{sec:supp-population-geometry}

{Diagonal attenuation can preserve the population PCA target exactly or move it
along the path $\bT_\alpha=\bSigma-\alpha\cdot\diag(\bSigma)$.  This section first
characterizes exact preservation, including the unequal-marginal example from
the main text, and then quantifies the population cost when preservation is
only approximate.}

\subsection{Exact preservation with unequal marginals}
\label{sec:supp-preserving-imbalance}

{Balanced marginal variances are sufficient but not necessary for exact
target preservation.  To verify the unequal design in
\eqref{eq:unequal-example}, note that its diagonal is
$D=\diag(4,4,2.5,2.5)$.  Let}
$u_1=(1,1,0,0)^\trans/\sqrt2$ and
$u_2=(0,0,1,1)^\trans/\sqrt2$.  The leading population subspace is
$\operatorname{span}\{u_1,u_2\}$, and $D$ preserves this subspace and its
complement.  {Direct calculation gives}
$\inf_{0\leq\alpha\leq1}\delta_\alpha=7/2-3/\sqrt2>1.37$, so the PCA target
{is retained along the full attenuation path.}

\begin{proof}[Proof of \Cref{prop:target-preservation}]
If
$P_* D P_*^\perp=0$, symmetry of $D$ implies that both
$\operatorname{range}(P_*)$ and its orthogonal complement are invariant
under $D$.  They are already invariant under $\bSigma$, so they are invariant
under $\bT_\alpha=\bSigma-\alpha D$.  This matrix is block diagonal in the
decomposition.  The condition $\delta_\alpha>0$ places every eigenvalue of the
target block above every eigenvalue of the residual block.  Its leading
rank-$d$ projector is therefore uniquely $P_*$.  Conversely, if
$P_*$ is the unique leading projector of $\bT_\alpha$, it is invariant
under both $\bT_\alpha$ and $\bSigma$.  Since $\alpha>0$, subtracting their
cross-block equations gives $P_* D P_*^\perp=0$; uniqueness gives
$\delta_\alpha>0$.
\end{proof}

\subsection{Population bias along the attenuation path}
\label{sec:supp-population-bias-path}

{To control the population cost, follow
$\bT_t=\bSigma-tD$, which begins at $P_{\bT_0}=P_*$.  If the cutoff eigengap
remains positive along this path, define}
\begin{equation}
 \label{eq:population-bias-bound}
 \underline g_\alpha
 :=\inf_{0\leq t\leq\alpha}
   \{\lambda_d(\bT_t)-\lambda_{d+1}(\bT_t)\}>0,
 \qquad
 \mathsf{B}_{\rm pop}(\alpha)
 \leq
 \frac{\alpha^2}{\underline g_\alpha}
 \sup_{0\leq t\leq\alpha}
 \norm{P_{\bT_t}^\perp DP_{\bT_t}}_{\F}^2.
\end{equation}
{Thus population movement remains small when the diagonal creates little
cross-subspace coupling along the path and the cutoff eigengap stays
separated.}

\begin{proof}[{Proof of the population-bias identity and
\eqref{eq:population-bias-bound}}]
We follow the attenuation path directly.  Write $P_t=P_{\bT_t}$ and
$B(t)=\tr\{\bSigma(P_*-P_t)\}$.  The positive cutoff gap makes the cluster
projector $P_t$ differentiable, even if eigenvalues repeat within the leading
or residual block.  At each $t$, choose orthonormal eigenbases
$\bT_t v_j(t)=\mu_j(t)v_j(t)$ for the two blocks.  The standard
spectral-projector derivative, which is basis invariant within repeated
eigenspaces, gives
\[
 P_t'
 =-\sum_{j\leq d<k}
 \frac{v_k(t)^\trans Dv_j(t)}{\mu_j(t)-\mu_k(t)}
 \{v_k(t)v_j(t)^\trans+v_j(t)v_k(t)^\trans\}.
\]
Because $P_t$ commutes with $\bT_t$,
$\tr(\bT_tP_t')=0$.  Using $\bSigma=\bT_t+tD$ therefore yields
\begin{align*}
 B'(t)
 &=-\tr(\bSigma P_t')
 =-t\tr(DP_t')
 =2t\sum_{j\leq d<k}
 \frac{\{v_k(t)^\trans Dv_j(t)\}^2}
      {\mu_j(t)-\mu_k(t)}.
\end{align*}
Since $B(0)=0$, integration gives the exact identity
\begin{equation}
 \label{eq:population-bias-path-identity}
 \mathsf B_{\rm pop}(\alpha)
 =2\int_0^\alpha t
 \sum_{j\leq d<k}
 \frac{\{v_k(t)^\trans Dv_j(t)\}^2}
      {\mu_j(t)-\mu_k(t)}\,dt.
\end{equation}
Every denominator is at least $\underline g_\alpha$, while the sum of the
squared numerators is
$\norm{(I_p-P_t)DP_t}_{\F}^2$.  Bounding the integrand by its supremum and
using $2\int_0^\alpha t\,dt=\alpha^2$ proves
\eqref{eq:population-bias-bound}.  In particular, if
$\bSigma u_k=\lambda_k u_k$, continuity of the integrand at zero gives
\[
 \mathsf B_{\rm pop}(\alpha)
 =\alpha^2\sum_{j\leq d<k}
 \frac{(u_k^\trans Du_j)^2}{\lambda_j-\lambda_k}
 +O(\alpha^3).
\]
This local expression shows explicitly why distant residual directions do not
require a separate tail penalty: their larger eigengaps reduce their rotation.
\end{proof}

{The next section uses this population decomposition in the expected-risk
theory: exact preservation removes the first term, while the target-moving
result retains it and quantifies the remaining sampling contribution around
$P_{\bT_\alpha}$.}

\section{Proofs of the finite-sample risk results}
\label{sec:supp-general-risk}

{The main text first uses individual retained-to-discarded rotations to
build intuition and then treats the full subspace.  Here we provide a common projector perturbation argument, which justifies both the target-preserving and target-moving risk theorems.  We then specialize the target-preserving coefficient to balanced marginal variances, where the sampling-leakage calculation becomes explicit.}

\subsection{A common clean-risk expansion}
\label{sec:supp-two-matrix}

{We first establish a two-matrix projector expansion that applies whether
the population target is preserved or moves.}

\begin{lemma}[Clean risk around a possibly moving target]
\label{lem:two-matrix-risk-expansion}
Let $A$ and $B$ be symmetric matrices, let $P=P_A$, and suppose that the cutoff
gap of $A$ is at least $\gamma>0$.  {Let $P^\perp=I_p-P$ denote the
complementary projector and set
$\mathbb X=\{X:X=P^\perp XP\}$.}  On $\mathbb X$, define
\begin{equation}\label{eq:supp-two-operators}
 \mathcal L_A(X)=X(PAP)-(P^\perp A P^\perp)X,
 \qquad
 \mathcal H_B(X)=X(PBP)-(P^\perp B P^\perp)X.
\end{equation}
For symmetric $F$, define
\begin{align}
 Y_{1,F}&=\mathcal L_A^{-1}(P^\perp FP),
 \label{eq:supp-moving-Y1}\\
 Y_{2,F}&=\mathcal L_A^{-1}\{P^\perp F P^\perp Y_{1,F}-Y_{1,F}PFP\}.
 \label{eq:supp-moving-Y2}
\end{align}
There is a finite constant $C_{A,B,d}$ such that every symmetric $F$ with
$\norm{F}_{\op}\leq\gamma/4$ satisfies
\begin{align}
 \tr\{B(P-P_{A+F})\}
 &=-2\langle P^\perp BP,Y_{1,F}\rangle_{\F}
   +\langle Y_{1,F},\mathcal H_B(Y_{1,F})\rangle_{\F} \nonumber\\
 &\quad
   -2\langle P^\perp BP,Y_{2,F}\rangle_{\F}
   +R_{A,B}^{\rm mov}(F),
 \label{eq:supp-moving-two-matrix-expansion}\\
 |R_{A,B}^{\rm mov}(F)|
 &\leq C_{A,B,d}\norm{F}_{\op}^3.
 \label{eq:supp-moving-two-matrix-remainder}
\end{align}
{If $P^\perp BP=0$, this reduces to}
\begin{align}
 \tr\{B(P-P_{A+F})\}
 &=\left\langle X_F,\mathcal H_B(X_F)\right\rangle_{\F}
   +R_{A,B}(F),
 \label{eq:supp-two-matrix-expansion}\\
 X_F&=Y_{1,F},
 \qquad
 |R_{A,B}(F)|\leq C_{A,B,d}\norm{F}_{\op}^3.
 \label{eq:supp-two-matrix-remainder}
\end{align}
The constant can be chosen uniformly over a compact family of pairs $(A,B)$
whose norms are bounded and whose relevant cutoff gaps are bounded below.
\end{lemma}

\begin{proof}[Proof of \Cref{lem:two-matrix-risk-expansion}]
{The gap makes $\mathcal L_A$ invertible on $\mathbb X$, with inverse
norm at most $1/\gamma$.  Weyl's inequality also keeps the leading cluster of
$A+F$ separated.  Its range is therefore the graph of a unique map
$Y:\operatorname{range}(P)\to\operatorname{range}(P^\perp)$ with
$\norm{Y}_{\op}=O(\norm{F}_{\op}/\gamma)$.  The block eigenvalue equation for
this graph is the Riccati equation}
\begin{equation}\label{eq:supp-riccati}
 \mathcal L_A(Y)
 =P^\perp FP+P^\perp F P^\perp Y-YPFP-YPF P^\perp Y.
\end{equation}
Applying $\mathcal L_A^{-1}$, substituting the first-order approximation into
the two quadratic terms, and using the preceding bound gives
\begin{equation}\label{eq:supp-graph-linearization}
 Y=Y_{1,F}+Y_{2,F}+O(\norm{F}_{\op}^3).
\end{equation}
All constants here depend only on the fixed matrices and the gap.

{Relative to
$\operatorname{range}(P)\oplus\operatorname{range}(P^\perp)$, the projector
onto the graph of $Y$ has blocks}
\begin{equation}\label{eq:supp-graph-projector-blocks}
 P_{A+F}
 =\begin{pmatrix}
 (I+Y^\trans Y)^{-1}&(I+Y^\trans Y)^{-1}Y^\trans\\
 Y(I+Y^\trans Y)^{-1}&Y(I+Y^\trans Y)^{-1}Y^\trans
 \end{pmatrix}.
\end{equation}
Expanding these blocks to second order and using
\eqref{eq:supp-graph-linearization} gives
\begin{align*}
 \tr\{B(P-P_{A+F})\}
 &=-2\langle P^\perp BP,Y_{1,F}+Y_{2,F}\rangle_{\F}
     +\langle Y_{1,F},\mathcal H_B(Y_{1,F})\rangle_{\F}
      +O(\norm{F}_{\op}^3),
\end{align*}
{This proves \eqref{eq:supp-moving-two-matrix-expansion}.  When
$P^\perp BP=0$, the
two off-diagonal terms vanish and
\eqref{eq:supp-two-matrix-expansion} follows.  The graph, Riccati, and inverse
bounds are uniform over the compact family stated in the lemma.}
\end{proof}

\subsection{Target-preserving expected risk}\label{sec:supp-fixed-p}

{We now specialize the common expansion to attenuation paths that preserve the population PCA target.}

\begin{proof}[Proof of \Cref{thm:fixed-p}]
For $\alpha\in\mathcal I$, put
\begin{equation}\label{eq:supp-general-score-decomposition}
 A_\alpha=\bT_\alpha,
 \qquad
 F_{n,\alpha}=E_n-\alpha\diag(E_n),
 \qquad
 E_n=\bS_n-\bSigma.
\end{equation}
By \cref{ass:uniform-target}, $A_\alpha$ and $\bSigma$ have the same
leading projector $P=P_*$, and the cutoff gap of $A_\alpha$ is at least
$\underline\delta_{\mathcal I}$.  The closure of
$\{A_\alpha:\alpha\in\mathcal I\}$ is compact; continuity of the block gap
and the strictly positive infimum preserve the same projector and gap bound on
this closure.  Hence
\cref{lem:two-matrix-risk-expansion} applies with one uniform remainder
constant whenever
$\norm{F_{n,\alpha}}_{\op}\leq
\underline\delta_{\mathcal I}/4$.

The linear rotation term in that lemma is
\begin{align}
 X_{n,\alpha}
 &=\mathcal L_\alpha^{-1}(P^\perp F_{n,\alpha}P)
  =\frac1n\sum_{i=1}^n
 \mathcal L_\alpha^{-1}Z_\alpha(\bx_i),
 \label{eq:supp-general-linear-rotation}
\end{align}
where the summands are independent and centered.  Cross terms therefore
vanish in expectation, giving the exact identity
\begin{equation}\label{eq:supp-general-leading-expectation}
 n\E\left\langle
 X_{n,\alpha},\mathcal H_{\bSigma}(X_{n,\alpha})
 \right\rangle_{\F}
 =\mathcal K(\alpha).
\end{equation}

For fixed $p$, Gaussian covariance concentration gives
$\E\norm{E_n}_{\op}^r=O(n^{-r/2})$ for every fixed $r$.  Since
$\norm{\diag(E_n)}_{\op}\leq\norm{E_n}_{\op}$ and
$0\leq\alpha\leq1$,
\begin{equation}\label{eq:supp-general-perturbation-bound}
 \sup_{\alpha\in\mathcal I}\norm{F_{n,\alpha}}_{\op}
 \leq2\norm{E_n}_{\op}.
\end{equation}
On the event
$\norm{E_n}_{\op}\leq\underline\delta_{\mathcal I}/8$, the cubic remainder
in \cref{lem:two-matrix-risk-expansion} thus has expectation
$O(n^{-3/2})$, uniformly in $\alpha$.  The complement has exponentially small
probability.  Both the clean excess risk and the quadratic leading term have
finite Gaussian moments, so Cauchy--Schwarz makes their contribution on that
complement exponentially smaller than $n^{-3/2}$.  Combining these facts with
\eqref{eq:supp-general-leading-expectation} gives
\[
 \sup_{\alpha\in\mathcal I}
 \left|\E\Rex(\widehat P_\alpha)
 -\frac{\mathcal K(\alpha)}n\right|=O(n^{-3/2}).
\]
Multiplying by $n$ and enlarging the constant to cover finitely many initial
sample sizes proves \eqref{eq:uniform-remainder}.
\end{proof}

\subsection{Total risk around a moving population target}
\label{sec:target-moving-risk}
\label{sec:supp-target-moving-proof}

{We now remove target preservation.  Let
$P_\alpha=P_{\bT_\alpha}$ and suppose that the}
cutoff gap of $\bT_\alpha$ remains positive over a compact interval
$\mathcal I\subset[0,1]$ containing zero:
\begin{equation}\label{eq:moving-gap}
 g_{\mathcal I}
 =\inf_{\alpha\in\mathcal I}
 \{\lambda_d(\bT_\alpha)-\lambda_{d+1}(\bT_\alpha)\}>0.
\end{equation}
{The theorem separates the deterministic cost of moving from $P_*$ from
the finite-sample reconstruction effect of sampling-induced rotation around
$P_\alpha$.}  The centered first-order sampling rotation
has mean zero.  At second order, both its variance and its average displacement
matter because clean reconstruction risk is measured relative to $P_*$ rather
than $P_\alpha$.
For one observation, write the centered attenuated fluctuation as
\begin{equation}\label{eq:moving-fluctuation}
 \Xi_\alpha(x)
 =xx^\trans-\bSigma
 -\alpha\diag(xx^\trans-\bSigma).
\end{equation}
On the cross-subspace maps
$\mathbb X_\alpha=\{X:X=P_\alpha^\perp XP_\alpha\}$, define
\begin{align}
 \mathcal L_\alpha^{\rm mov}(X)
 &=X(P_\alpha\bT_\alpha P_\alpha)
   -(P_\alpha^\perp\bT_\alpha P_\alpha^\perp)X,
 \label{eq:moving-gap-operator}\\
 \mathcal H_{\bSigma,\alpha}(X)
 &=X(P_\alpha\bSigma P_\alpha)
   -(P_\alpha^\perp\bSigma P_\alpha^\perp)X.
 \label{eq:moving-clean-operator}
\end{align}
The first two graph perturbations generated by $\Xi_\alpha(x)$ are
\begin{align}
 Y_{1,\alpha}(x)
 &=\{\mathcal L_\alpha^{\rm mov}\}^{-1}
   \{P_\alpha^\perp\Xi_\alpha(x)P_\alpha\},
 \label{eq:moving-Y1}\\
 Y_{2,\alpha}(x)
 &=\{\mathcal L_\alpha^{\rm mov}\}^{-1}
 \left\{
 P_\alpha^\perp\Xi_\alpha(x)P_\alpha^\perp Y_{1,\alpha}(x)
 -Y_{1,\alpha}(x)P_\alpha\Xi_\alpha(x)P_\alpha
 \right\}.
 \label{eq:moving-Y2}
\end{align}
Define the signed sampling coefficient
\begin{equation}\label{eq:moving-risk-constant}
 \mathcal K_{\rm move}(\alpha)
 =\E\left[
 \left\langle
 Y_{1,\alpha},
 \mathcal H_{\bSigma,\alpha}(Y_{1,\alpha})
 \right\rangle_{\F}
 -2\left\langle
 P_\alpha^\perp\bSigma P_\alpha,Y_{2,\alpha}
 \right\rangle_{\F}
 \right].
\end{equation}
The second term is absent under target preservation.  When the population
target moves, it records whether the average second-order sampling displacement
points toward or away from the clean PCA subspace; consequently,
$\mathcal K_{\rm move}(\alpha)$ need not be nonnegative by itself.

\begin{theorem}[Total risk with population target movement]
\label{thm:target-moving-risk}
Under \cref{ass:gaussian-sampling,ass:target-gap} and
\eqref{eq:moving-gap}, there are finite constants $C$ and $N$ such that
\begin{equation}\label{eq:moving-total-risk}
 \sup_{\alpha\in\mathcal I}
 \left|
 \E\Rex(\widehat P_\alpha)
 -\mathsf B_{\rm pop}(\alpha)
 -\frac{\mathcal K_{\rm move}(\alpha)}{n}
 \right|
 \leq \frac{C}{n^{3/2}},
 \qquad n\geq N.
\end{equation}
\end{theorem}

\begin{proof}[Proof of \Cref{thm:target-moving-risk}]
For $\alpha\in\mathcal I$, write
\begin{equation}\label{eq:supp-moving-decomposition}
 A_\alpha=\bT_\alpha,
 \qquad
 F_{n,\alpha}=\frac1n\sum_{i=1}^n\Xi_\alpha(\bx_i),
 \qquad
 P_\alpha=P_{A_\alpha}.
\end{equation}
The excess risk has the exact decomposition
\begin{equation}\label{eq:supp-moving-exact-decomposition}
 \Rex(\widehat P_\alpha)
 =\mathsf B_{\rm pop}(\alpha)
 +\tr\{\bSigma(P_\alpha-P_{A_\alpha+F_{n,\alpha}})\}.
\end{equation}
Apply \cref{lem:two-matrix-risk-expansion} to the second term with
$A=A_\alpha$, $B=\bSigma$, and $P=P_\alpha$.  Let
$Y_{1,n,\alpha}$ and $Y_{2,n,\alpha}$ denote the two graph terms in
\cref{eq:supp-moving-Y1,eq:supp-moving-Y2}.  Linearity gives
\begin{equation}\label{eq:supp-moving-Y1-average}
 Y_{1,n,\alpha}
 =\frac1n\sum_{i=1}^nY_{1,\alpha}(\bx_i),
 \qquad
 \E Y_{1,n,\alpha}=0.
\end{equation}
The quadratic graph term expands into a double sum.  Independence and
centering remove every summand with distinct observation indices, leaving
\begin{equation}\label{eq:supp-moving-Y2-average}
 \E Y_{2,n,\alpha}
 =\frac1n\E Y_{2,\alpha}(\bx).
\end{equation}
The same argument applied to the clean quadratic term gives
\begin{equation}\label{eq:supp-moving-quadratic-average}
 \E\left\langle
 Y_{1,n,\alpha},
 \mathcal H_{\bSigma,\alpha}(Y_{1,n,\alpha})
 \right\rangle_{\F}
 =\frac1n\E\left\langle
 Y_{1,\alpha},
 \mathcal H_{\bSigma,\alpha}(Y_{1,\alpha})
 \right\rangle_{\F}.
\end{equation}
The expectation of the linear graph term is zero.  Combining
\cref{eq:supp-moving-Y2-average,eq:supp-moving-quadratic-average} with
\eqref{eq:supp-moving-two-matrix-expansion} therefore identifies the complete
second-order expectation as $\mathcal K_{\rm move}(\alpha)/n$.

It remains to control the remainder uniformly.  The gap condition
\eqref{eq:moving-gap} and compactness of $\{A_\alpha:\alpha\in\mathcal I\}$
give one graph-expansion constant for the whole interval.  Moreover,
\begin{equation}\label{eq:supp-moving-perturbation-bound}
 \sup_{\alpha\in\mathcal I}\norm{F_{n,\alpha}}_{\op}
 \leq2\norm{\bS_n-\bSigma}_{\op}.
\end{equation}
The Gaussian moment and complement-event argument used in
\eqref{eq:supp-general-perturbation-bound} consequently bounds the expected
cubic remainder by $O(n^{-3/2})$, uniformly in $\alpha$.  Substitution into
\eqref{eq:supp-moving-exact-decomposition} proves
\eqref{eq:moving-total-risk}.  
{When $P_\alpha=P_*$, we have
$P_\alpha^\perp\bSigma P_\alpha=0$; the second graph contribution vanishes and}
$\mathcal K_{\rm move}(\alpha)=\mathcal K(\alpha)$, recovering
\cref{thm:fixed-p}.
\end{proof}

The theorem gives a direct comparison with PCA:
\begin{equation}\label{eq:moving-benefit-comparison}
 \E\Rex(\widehat P_\alpha)-\E\Rex(\widehat P_0)
 =\mathsf B_{\rm pop}(\alpha)
 +\frac{\mathcal K_{\rm move}(\alpha)-\mathcal K(0)}{n}
 +O(n^{-3/2}).
\end{equation}
Thus a fixed strength is beneficial to first order at sample size $n$ when
\begin{equation}\label{eq:moving-benefit-condition}
 n\mathsf B_{\rm pop}(\alpha)
 <\mathcal K(0)-\mathcal K_{\rm move}(\alpha).
\end{equation}
If $\mathsf B_{\rm pop}(\alpha)>0$ for a fixed $\alpha$, this inequality
eventually fails as $n$ grows.  More generally, if
$\mathsf B_{\rm pop}$ has the unique zero $\alpha=0$ on $\mathcal I$, every
oracle minimizer of the expected risk converges to zero.  This is the regime in
which attenuation is a genuinely finite-sample correction rather than a fixed
population transformation.

The small-$\alpha$ form makes the return to PCA more explicit.  With
$\bSigma u_j=\lambda_ju_j$, the population term satisfies
\begin{equation}\label{eq:moving-bias-local}
 \mathsf B_{\rm pop}(\alpha)
 =\alpha^2\mathsf C_D+O(\alpha^3),
 \qquad
 \mathsf C_D=
 \sum_{j\leq d<k}
 \frac{(u_k^\trans Du_j)^2}{\lambda_j-\lambda_k}.
\end{equation}
If $\mathsf C_D>0$ and
$\mathsf R_{\rm move}:=-\mathcal K_{\rm move}'(0)>0$, then the leading
approximation in \eqref{eq:moving-benefit-comparison} is minimized locally at
\begin{equation}\label{eq:moving-local-oracle}
 \widetilde\alpha_n
 =\frac{\mathsf R_{\rm move}}{2\mathsf C_D\,n}
 +o(n^{-1}).
\end{equation}
Equation \eqref{eq:moving-local-oracle} concerns the minimizer of the displayed leading approximation; the exact oracle convergence to zero follows directly from the uniform expansion above.

{We next return to balanced marginals, where the target-preserving risk coefficient can be evaluated explicitly.}

\subsection{Balanced marginals and the explicit risk curve}
\label{sec:supp-balanced-risk}

{We now make the target-preserving risk coefficient explicit under
balanced marginal variances.}

\subsubsection{Pairwise form under balanced marginals}
\label{sec:supp-balanced-pairs}

{Under balanced marginals, recall the pairwise variables
$a_{jk}(\bx)$ and $b_{jk}(\bx)$ from
\eqref{eq:single-observation-scores-main}.  To lighten notation in this
subsection, write $a_{jk}=a_{jk}(\bx)$ and $b_{jk}=b_{jk}(\bx)$.  Attenuation
replaces the sample average of $a_{jk}$ by that of
$a_{jk}-\alpha b_{jk}$.  For the moment calculations below, put}
\begin{equation}\label{eq:overlap-defs}
 w_{jk}=u_j\Had u_k,
 \qquad
 s_{jk}=\norm{w_{jk}}_2^2,
 \qquad
 h_{jk}=2w_{jk}^\trans(\bSigma\Had\bSigma)w_{jk}.
\end{equation}
Here $s_{jk}$ measures how strongly the two population directions load on the
same observed coordinates.  Isserlis' Gaussian fourth-moment identity
\citep{isserlis1918formula} gives
\begin{equation}\label{eq:single-observation-moments}
 \var(a_{jk})=\lambda_j\lambda_k,
 \quad
 \Cov(a_{jk},b_{jk})=2\lambda_j\lambda_k s_{jk},
 \quad
 \var(b_{jk})=h_{jk}.
\end{equation}

\begin{proof}[Proof of \Cref{lem:balanced-coupling-risk}]
{Fix $\alpha\in[0,1]$ and write
$F_{n,\alpha}=\bT_{n,\alpha}-\bT_\alpha$ and
$\Delta P_\alpha=\widehat P_\alpha-P_*$.  Gaussian covariance concentration
and the fixed cutoff gap, together with the Davis--Kahan theorem, give}
\begin{equation}\label{eq:supp-balanced-perturbation-orders}
 \norm{F_{n,\alpha}}_{\op}=O_{\mathbb P}(n^{-1/2}),
 \qquad
 \norm{\Delta P_\alpha}_{\op}=O_{\mathbb P}(n^{-1/2}).
\end{equation}
{Because $\widehat P_\alpha$ is a spectral projector of
$\bT_{n,\alpha}$, the two matrices commute.  Substituting
$\bT_{n,\alpha}=\bT_\alpha+F_{n,\alpha}$ and
$\widehat P_\alpha=P_*+\Delta P_\alpha$, and using
$\bT_\alpha P_*=P_*\bT_\alpha$, yields}
\begin{align}
0
&=\bT_{n,\alpha}\widehat P_\alpha
  -\widehat P_\alpha\bT_{n,\alpha}\nonumber\\
&=(\bT_\alpha P_*-P_*\bT_\alpha)
  +(\bT_\alpha\Delta P_\alpha-\Delta P_\alpha\bT_\alpha)\nonumber\\
&\quad +(F_{n,\alpha}P_*-P_*F_{n,\alpha})
  +(F_{n,\alpha}\Delta P_\alpha-\Delta P_\alpha F_{n,\alpha})\nonumber\\
&=\bT_\alpha\Delta P_\alpha-\Delta P_\alpha\bT_\alpha
  +F_{n,\alpha}P_*-P_*F_{n,\alpha}
  +O_{\mathbb P}(n^{-1}).
\label{eq:supp-balanced-projector-commutation}
\end{align}
{The first parenthesis vanishes because $P_*$ is a spectral projector of
$\bT_\alpha$; the last is $O_{\mathbb P}(n^{-1})$ by
\eqref{eq:supp-balanced-perturbation-orders}.  Under balance,
$\bT_\alpha u_r=(\lambda_r-\alpha\tau)u_r$.  Hence}
\begin{align*}
u_k^\trans(\bT_\alpha\Delta P_\alpha
 -\Delta P_\alpha\bT_\alpha)u_j
&=-(\lambda_j-\lambda_k)
  u_k^\trans\Delta P_\alpha u_j,\\
u_k^\trans(F_{n,\alpha}P_*-P_*F_{n,\alpha})u_j
&=u_k^\trans F_{n,\alpha}u_j,
\end{align*}
{where the second equality uses $P_*u_j=u_j$ and
$u_k^\trans P_*=0$.  Premultiplying
\eqref{eq:supp-balanced-projector-commutation} by $u_k^\trans$ and
postmultiplying by $u_j$ therefore gives}
\begin{equation*}
 0=-(\lambda_j-\lambda_k)
 u_k^\trans\Delta P_\alpha u_j
 +u_k^\trans F_{n,\alpha}u_j+O_{\mathbb P}(n^{-1}).
\end{equation*}
{{Since
$u_k^\trans F_{n,\alpha}u_j
=\overline a_{jk}-\alpha\overline b_{jk}$, solving the preceding identity and
using \eqref{eq:balanced-leading-rotation},
then summing over the finitely many pairs $j\leq d<k$, proves
\eqref{eq:balanced-rotation-relation}; the matrix remainder is
$O_{\mathbb P}(n^{-1})$ in Frobenius norm because $p$ is fixed.}}

{Finally, adding the scalar matrix $\alpha\tau I_p$ does not change
spectral projectors.  Hence $\widehat P_\alpha$ is also the leading projector
of $\bSigma+F_{n,\alpha}$.  We may therefore apply
\cref{lem:two-matrix-risk-expansion} with $A=B=\bSigma$ and $P=P_*$.  Because
$P_*^\perp\bSigma P_*=0$, the simplified conclusion
\eqref{eq:supp-two-matrix-expansion} applies.  To evaluate its quadratic term,
expand the cross-block perturbation in the population eigenbasis:}
\begin{equation*}
 P_*^\perp F_{n,\alpha}P_*
 =\sum_{j\leq d<k}
   (u_k^\trans F_{n,\alpha}u_j)u_ku_j^\trans.
\end{equation*}
{For each basis matrix $u_ku_j^\trans$, the operator in
\eqref{eq:supp-two-operators} satisfies}
\begin{equation*}
 \mathcal L_{\bSigma}(u_ku_j^\trans)
 =\mathcal H_{\bSigma}(u_ku_j^\trans)
 =(\lambda_j-\lambda_k)u_ku_j^\trans.
\end{equation*}
{Consequently, writing $f_{kj}=u_k^\trans F_{n,\alpha}u_j$,}
\begin{align*}
 X_{F_{n,\alpha}}
 &=\mathcal L_{\bSigma}^{-1}(P_*^\perp F_{n,\alpha}P_*)
   =\sum_{j\leq d<k}
     \frac{f_{kj}}{\lambda_j-\lambda_k}u_ku_j^\trans,\\
 \left\langle X_{F_{n,\alpha}},
 \mathcal H_{\bSigma}(X_{F_{n,\alpha}})\right\rangle_{\F}
 &=\sum_{j\leq d<k}
   \frac{f_{kj}^2}{\lambda_j-\lambda_k},
\end{align*}
{where the last identity uses the Frobenius orthonormality of
$\{u_ku_j^\trans:j\leq d<k\}$.  Thus
\eqref{eq:supp-two-matrix-expansion} gives}
\begin{equation*}
 \Rex(\widehat P_\alpha)
 =\sum_{j\leq d<k}
   \frac{(u_k^\trans F_{n,\alpha}u_j)^2}{\lambda_j-\lambda_k}
   +O_{\mathbb P}(\norm{F_{n,\alpha}}_{\op}^3).
\end{equation*}
{Using \eqref{eq:supp-balanced-perturbation-orders} and the same
cross-entry identity proves \eqref{eq:balanced-local-risk}.}
\end{proof}

{The moment identities above show how attenuation changes each
retained-to-omitted coupling before these contributions are aggregated across
the rank cutoff.}

Aggregating across the rank cutoff gives
\begin{align}
 \Asig
 &=\sum_{j\leq d<k}
 \frac{\lambda_j\lambda_k s_{jk}}{\lambda_j-\lambda_k},
 &
 \Bsig
 &=\sum_{j\leq d<k}
 \frac{h_{jk}}{\lambda_j-\lambda_k},
 \label{eq:ABconstants}\\
 G(\alpha)
 &=\sum_{j\leq d<k}
 \frac{\lambda_j\lambda_k
 -4\alpha\lambda_j\lambda_k s_{jk}+\alpha^2h_{jk}}
 {\lambda_j-\lambda_k}.
 \label{eq:Galpha}
\end{align}
Consequently, $\mathcal K(\alpha)=G(\alpha)$,
$\mathsf{C}_{\bSigma}=2\Asig$, and $\mathsf{V}_{\bSigma}=\Bsig$.  These identities
connect the readable main-text tradeoff to the pairwise quantities used in
the proofs below.

\subsubsection{From pairwise moments to the risk curve}
\label{sec:supp-balanced-proof}

\begin{proof}[Proof of \Cref{thm:balanced-risk}]
Under balance,
\[
 \bT_{n,\alpha}=\bSigma-\alpha\tau I_p+E_{n,\alpha},
 \qquad
 E_{n,\alpha}=E_n-\alpha\diag(E_n),
 \qquad E_n=\bS_n-\bSigma.
\]
The scalar shift preserves all population eigenvectors and gaps.  Equal-rank
projectors have equal trace, so replacing $\bSigma$ by
$\bSigma-\alpha\tau I_p$ also leaves the clean excess risk unchanged.
{After removing this shift, apply \Cref{lem:two-matrix-risk-expansion} with $A=B=\bSigma$ and $F=E_{n,\alpha}$.  In the eigenbasis of $\bSigma$, its quadratic term is the
sum below.  Thus, on
$\norm{E_{n,\alpha}}_{\op}\leq\delta/4$,}
\begin{equation}\label{eq:supp-projector-expansion}
 \Rex(\widehat P_\alpha)
 =\sum_{j\leq d<k}
 \frac{(u_j^\trans E_{n,\alpha}u_k)^2}
 {\lambda_j-\lambda_k}
 +\operatorname{Rem}_{n,\alpha},
 \qquad
 \abs{\operatorname{Rem}_{n,\alpha}}\leq C_{\rm rem}\norm{E_{n,\alpha}}_{\op}^3,
\end{equation}
{Here $C_{\rm rem}$ depends only on the fixed spectrum and cutoff gap.
The formula is invariant to basis rotations inside a repeated-eigenvalue block,
because the corresponding squared cross-block entries share the same
denominator.}

Because $u_j^\trans E_{n,\alpha}u_k$ is the sample average of
$a_{jk}-\alpha b_{jk}$, the moment identities in
\eqref{eq:single-observation-moments} give
\begin{align}
 n\E(u_j^\trans E_{n,\alpha}u_k)^2
 &=\var(a_{jk}-\alpha b_{jk})
  =\lambda_j\lambda_k
 -4\alpha\lambda_j\lambda_k s_{jk}+\alpha^2h_{jk}.
 \label{eq:supp-entry-variance}
\end{align}

Equation \eqref{eq:supp-entry-variance} gives the expectation of the quadratic
term in \eqref{eq:supp-projector-expansion} as $G(\alpha)/n$.  The uniform
good-event and complement argument in the proof of \cref{thm:fixed-p} applies
unchanged because
$\sup_{0\leq\alpha\leq1}\norm{E_{n,\alpha}}_{\op}
\leq2\norm{E_n}_{\op}$.  Consequently,
\begin{equation*}
 \sup_{0\leq\alpha\leq1}
 \left|\E\Rex(\widehat P_\alpha)-\frac{G(\alpha)}n\right|
 =O(n^{-3/2}),
\end{equation*}
which proves \eqref{eq:balanced-uniform-remainder} after enlarging its constant for
finitely many initial sample sizes.  For a fixed $\alpha$ with
$\Delta_\alpha=G(0)-G(\alpha)>0$,
\begin{align*}
 n\{\E\Rex(\widehat P_0)-\E\Rex(\widehat P_\alpha)\}
 &\geq \Delta_\alpha-\frac{2C}{\sqrt n},
\end{align*}
which yields an eventual finite-sample improvement whenever the theoretical
risk curve is smaller than at PCA.  Summing \eqref{eq:supp-entry-variance} over
cross-boundary pairs gives \eqref{eq:Galpha}.  These perturbation and
concentration steps are
consistent with classical and nonasymptotic PCA analyses
\citep{anderson1963asymptotic,koltchinskii2017concentration,
reiss2020nonasymptotic}.

It remains to verify strict positivity and optimize the quadratic.  If every
$s_{jk}$ across the cutoff were zero, then every observed coordinate would
lie entirely in either the selected or rejected eigenspace.  Its variance
would consequently be at least $\lambda_d$ in the first case and at most
$\lambda_{d+1}$ in the second.  Both kinds of coordinates must occur, which
contradicts their common variance and the positive cutoff gap.  Hence
$\Asig>0$.  For a pair with $s_{jk}>0$, $w_{jk}\ne0$; positive definiteness
of $\bSigma\Had\bSigma$ then gives $h_{jk}>0$, so $\Bsig>0$.  Expanding
$G(\alpha)-G(0)$ proves \eqref{eq:risk-difference}.  This completes the proof
of \cref{thm:balanced-risk}. Differentiating the same strictly convex
quadratic and clipping its minimizer to the two closed ranges then gives
\eqref{eq:two-oracles}.
\end{proof}

\paragraph{The flat-spike benchmark.}
This standard spiked-covariance model gives a particularly transparent positive
case.  For the rank-one flat spike
$\bSigma=\sigma^2I_p+\theta uu^\trans$ with $u=p^{-1/2}\ones$, every residual
eigenvector $u_k$ satisfies
\begin{equation}\label{eq:flat-spike-terms}
 s_{1k}=\frac1p,
 \qquad
 h_{1k}=\frac{2}{p}
 \left(\sigma^4+\frac{2\sigma^2\theta}{p}\right),
 \qquad
 \frac{\mathsf C_{\bSigma}}{\mathsf V_{\bSigma}}
 =\frac{2\Asig}{\Bsig}
 =\frac{\sigma^2+\theta}{\sigma^2+2\theta/p}\geq1.
\end{equation}
The derivative $-4\Asig+2\alpha\Bsig$ is therefore nonpositive for every
$0\leq\alpha\leq1$, proving the flat-spike statement in the main text.

\section{Robustness and validation selection}
\label{sec:supp-extensions-practical}

{
This section extends the balanced-marginal risk theorem beyond Gaussian sampling and gives a formal guarantee for selecting attenuation by held-out reconstruction while retaining PCA as a candidate.
}

\subsection{Risk expansion beyond Gaussian sampling}
\label{sec:supp-robustness}

The preceding risk theory uses Gaussian fourth moments to obtain an explicit
sign.  The perturbation mechanism itself is more general: it depends on the
actual covariance between each rotation fluctuation and its diagonal
component.  The next result records this robustness after the Gaussian theory
has been completed.

For a distribution $\nu$ with covariance $\bSigma$, define
\begin{equation}\label{eq:moment-constants}
 \Cnu=\sum_{j\leq d<k}
 \frac{\Cov(a_{jk},b_{jk})}{\lambda_j-\lambda_k},
 \qquad
 \Vnu=\sum_{j\leq d<k}
 \frac{\var(b_{jk})}{\lambda_j-\lambda_k},
\end{equation}
where $a_{jk}$ and $b_{jk}$ are given in
\eqref{eq:single-observation-scores-main}.

\begin{theorem}[Moment extension of the attenuation risk]
\label{thm:moment-extension}
Let $\bx_1,\ldots,\bx_n$ be independent copies of a mean-zero random vector
with covariance $\bSigma$, $\E\norm{\bx}_2^8<\infty$, and balanced marginal
variances.  Keep $p,d$ fixed and assume a positive cutoff gap.  With
$a_{jk},b_{jk},\Cnu$, and $\Vnu$ as in
\eqref{eq:single-observation-scores-main} and \eqref{eq:moment-constants}, define
\begin{align}
 \Gnu(0)
 &=\sum_{j\leq d<k}
 \frac{\var(a_{jk})}{\lambda_j-\lambda_k},
 \nonumber\\
 \Gnu(\alpha)
 &=\Gnu(0)-2\alpha\Cnu+\alpha^2\Vnu.
 \label{eq:moment-risk-quadratic}
\end{align}
For every $0\leq\bar\alpha\leq1$, there are finite constants $K_\nu,N_\nu$
such that
\begin{equation}\label{eq:moment-uniform-remainder}
 \sup_{0\leq\alpha\leq\bar\alpha}
\left|n\E\Rex(\widehat P_\alpha)-\Gnu(\alpha)\right|
 \leq\frac{K_\nu}{\sqrt n},
 \qquad n\geq N_\nu.
\end{equation}
\end{theorem}

\begin{proof}[Proof of \Cref{thm:moment-extension}]
For a retained--discarded pair, both quantities in
\eqref{eq:single-observation-scores-main} are centered and
\[
 u_j^\trans E_{n,\alpha}u_k
 =\frac1n\sum_{i=1}^n(a_{jk,i}-\alpha b_{jk,i}).
\]
Consequently,
\begin{equation}\label{eq:supp-general-entry-variance}
 n\E(u_j^\trans E_{n,\alpha}u_k)^2
 =\var(a_{jk})-2\alpha\Cov(a_{jk},b_{jk})
 +\alpha^2\var(b_{jk}).
\end{equation}
Summing \eqref{eq:supp-general-entry-variance} in the leading quadratic of
\eqref{eq:supp-projector-expansion} gives exactly
$\Gnu(\alpha)/n$.

The eighth-moment condition and fixed dimension imply
$\E\norm{E_n}_{\op}^r=O(n^{-r/2})$ for $r=3,4$.  This follows by applying a
fixed-dimensional Rosenthal inequality to the entries of
$\bx_i\bx_i^\trans-\bSigma$ and then using equivalence of matrix norms.  Since
$\norm{E_{n,\alpha}}_{\op}\leq2\norm{E_n}_{\op}$ uniformly on
$0\leq\alpha\leq1$, the cubic remainder in
\eqref{eq:supp-projector-expansion} contributes $O(n^{-3/2})$ on the
perturbative event.  Markov's inequality with the fourth moment gives
probability $O(n^{-2})$ to its complement.  Excess risk is bounded, and the
leading quadratic on that complement is controlled by
$K\norm{E_n}_{\op}^2$; Cauchy--Schwarz makes both complementary contributions
$O(n^{-2})$.  Multiplication by $n$ proves the uniform bound
\eqref{eq:moment-uniform-remainder}.
\end{proof}

{For an interval $[a,b]$, let
$\Pi_{[a,b]}(t)=\min\{b,\max\{a,t\}\}$ denote Euclidean projection, or
clipping, onto that interval.}  The quadratic also gives the practical consequences.  If $\Vnu=0$, every
$b_{jk}$ is almost surely zero, so Cauchy--Schwarz forces $\Cnu=0$ and
attenuation leaves the leading risk unchanged.  Otherwise, $\Cnu>0$ implies
$\Vnu>0$; the direct leading-risk oracle is
$\Pi_{[0,1]}(\Cnu/\Vnu)$, and the mask-path oracle is its limiting projection
onto $[0,1/2]$.  A fixed level improves the leading risk whenever
$0<\alpha<2\Cnu/\Vnu$, subject also to $\alpha<1/2$ on the mask path.  More
precisely, if
$\Delta_{\nu,\alpha}=\Gnu(0)-\Gnu(\alpha)>0$, then
\begin{equation}\label{eq:moment-eventual-threshold}
 n\geq N_\nu
 \quad\text{and}\quad
 n>\left(\frac{2K_\nu}{\Delta_{\nu,\alpha}}\right)^2
 \quad\Longrightarrow\quad
 \E\Rex(\widehat P_\alpha)<\E\Rex(\widehat P_0).
\end{equation}
\subsection{Validation selection with a PCA fallback}
\label{sec:supp-validation-fallback}

The attenuation path contains ordinary PCA at $\alpha=0$.  The next result
turns this nesting into a finite-sample safety statement for a simple
sample-splitting implementation.  Candidate projectors are fitted once on a
training sample and then selected on an independent validation sample; no
post-selection refitting is needed for the guarantee.

\begin{proposition}[Independent-validation fallback guarantee]
\label{prop:validation-fallback}
Let $\mathcal A\subset[0,1]$ be a finite set containing $0$, and condition on
rank-$d$ projectors $\{\widehat P_\alpha:\alpha\in\mathcal A\}$ fitted from a
training sample.  Let $Y_1,\ldots,Y_{n_v}$ be an independent validation sample
from $\mathcal N(0,\bSigma)$ and define
\begin{equation}\label{eq:validation-risk}
 \widehat{\cR}_v(P)=\frac1{n_v}\sum_{i=1}^{n_v}
 \norm{Y_i-PY_i}_2^2,
 \qquad
 \widehat\alpha_v\in\argmin_{\alpha\in\mathcal A}
 \widehat{\cR}_v(\widehat P_\alpha).
\end{equation}
For $0<\eta<1$, put $M=|\mathcal A|$, $t=\log(2M/\eta)$, and
\begin{equation}\label{eq:validation-epsilon}
 \varepsilon_v(t)=
 2\norm{\bSigma}_{\F}\sqrt{\frac{t}{n_v}}
 +2\norm{\bSigma}_{\op}\frac{t}{n_v}.
\end{equation}
With conditional probability at least $1-\eta$,
\begin{equation}\label{eq:validation-fallback}
 \cR(\widehat P_{\widehat\alpha_v})
 \leq\min_{\alpha\in\mathcal A}\cR(\widehat P_\alpha)
      +2\varepsilon_v(t)
 \leq\cR(\widehat P_0)+2\varepsilon_v(t).
\end{equation}
Thus the selected estimator approaches an exact PCA fallback as the validation
size grows, while retaining any improvement available elsewhere on the path.
\end{proposition}

\begin{proof}[Proof of \Cref{prop:validation-fallback}]
Conditional on the training sample, fix a candidate $P$ and write
$Y_i=\bSigma^{1/2}Z_i$ with $Z_i\sim\mathcal N(0,I_p)$.  Then
\begin{equation*}
 \norm{Y_i-PY_i}_2^2=Z_i^\trans B_PZ_i,
 \qquad
 B_P=\bSigma^{1/2}(I_p-P)\bSigma^{1/2}\succeq0,
\end{equation*}
and $\E(Z_i^\trans B_PZ_i)=\tr(B_P)=\cR(P)$.  The Gaussian quadratic-form
inequality of \citet{laurent2000adaptive}, applied to the block-diagonal
average of the $n_v$ forms, gives
\begin{equation*}
 \Pr\!\left(
 \left|\widehat{\cR}_v(P)-\cR(P)\right|>
 2\norm{B_P}_{\F}\sqrt{\frac{t}{n_v}}
 +2\norm{B_P}_{\op}\frac{t}{n_v}
 \right)\leq2e^{-t}.
\end{equation*}
Because $0\preceq I_p-P\preceq I_p$, we have
$0\preceq B_P\preceq\bSigma$.  Eigenvalue monotonicity under the Loewner
order therefore gives $\norm{B_P}_{\F}\leq\norm{\bSigma}_{\F}$ and
$\norm{B_P}_{\op}\leq\norm{\bSigma}_{\op}$.  A union bound over the $M$
candidates therefore makes all validation risks uniformly accurate to
$\varepsilon_v(t)$ with probability at least $1-\eta$.  On this event, if
$\alpha^*\in\argmin_{\alpha\in\mathcal A}\cR(\widehat P_\alpha)$, then
\begin{align*}
 \cR(\widehat P_{\widehat\alpha_v})
 &\leq\widehat{\cR}_v(\widehat P_{\widehat\alpha_v})
       +\varepsilon_v(t)
  \leq\widehat{\cR}_v(\widehat P_{\alpha^*})+\varepsilon_v(t)
 \leq\cR(\widehat P_{\alpha^*})+2\varepsilon_v(t).
\end{align*}
Since $0\in\mathcal A$, the final inequality in
\eqref{eq:validation-fallback} follows.
\end{proof}

{The result is stated for independent validation because this permits a
direct finite-sample guarantee.  The real-data studies use the data-efficient
analogue: all candidate methods share the same grouped cross-validation folds
within each fixed fitting budget, and the selected estimator is then refit on
that budget.}

\section{Alternative masking and reconstruction objectives}\label{sec:supp-objective-boundaries}

{This section clarifies how the full-output orthogonal-projector
objective differs from masked-target reconstruction and unconstrained linear
autoencoders.  These comparisons delimit the estimator studied in the main
text; they are not needed for its risk theory.}

\subsection{The general block masked-target objective}

{The canonical masked-target projector loss is}
\begin{equation}\label{eq:masked-target-objective}
  \mathcal L_{\rm MT}(P)
  =\frac1n\sum_{i=1}^n
  \E_W\norm{(I_p-W)(\bx_i-PW\bx_i)}_2^2,
  \qquad P\in\cP_d.
\end{equation}
{It evaluates reconstruction error only on coordinates hidden from the
encoder, whereas the full-output loss in \eqref{eq:mae-objective} evaluates
error on the entire clean vector.}

{The matched empirical comparison in
\cref{sec:supp-masked-target-comparison} includes both coordinate masks and
masks that hide contiguous coordinate groups.}  Let
$\mathcal G=\{G_1,\ldots,G_G\}$ be a partition of the coordinates.  For a
matrix $A$, let $\mathcal D_{\mathcal G}(A)$ retain its within-group diagonal
blocks and put
$\mathcal O_{\mathcal G}(A)=A-\mathcal D_{\mathcal G}(A)$.  A group mask has
the form $W=\operatorname{blockdiag}(w_1I_{G_1},\ldots,w_GI_{G_G})$, where the
$w_g$ are independent Bernoulli-$q$ variables.  Write $R=I-W$ and
$D_P=\mathcal D_{\mathcal G}(P)$.

The calculation used for \cref{prop:reduction} also gives the full-output
group-masked matrix
\begin{equation}\label{eq:group-full-output-score}
 \bT_{n,\alpha,\mathcal G}
 =\bS_n-\alpha\mathcal D_{\mathcal G}(\bS_n),
 \qquad \alpha=\frac{1-q}{2-q}.
\end{equation}
{Thus group masking attenuates within-group covariance blocks rather than
only coordinate variances.  The masked-target loss has a different reduction.}

\begin{proposition}[Exact symmetric masked-target objective]
\label{prop:masked-target-general}
For $m=1-q$ and every $P\in\cP_d$, the group version of
\eqref{eq:masked-target-objective} satisfies
\begin{align}
 \mathcal L_{\rm MT}(P)
 &=m\tr(\bS_n)+mq\left[
 -2\inner{\bS_n}{\mathcal O_{\mathcal G}(P)}
 +\inner{\bS_n}{\mathcal C_{q,\mathcal G}(P)}
 \right],
 \label{eq:masked-target-general-reduction}\\
 \mathcal C_{q,\mathcal G}(P)
 &=D_P-D_P^2
 +q\,\mathcal O_{\mathcal G}\{P-D_PP-PD_P\}.
 \label{eq:masked-target-general-correction}
\end{align}
\end{proposition}

\begin{proof}[Proof of \Cref{prop:masked-target-general}]
Expand
\begin{equation}\label{eq:supp-mt-expand}
 \norm{R(x-PWx)}_2^2
 =x^\trans Rx-2x^\trans RPWx+x^\trans WPRPWx.
\end{equation}
For groups $g,h$, $\E\{(1-w_g)w_h\}$ is zero when $g=h$ and equals $mq$
otherwise.  Hence
\begin{equation}\label{eq:supp-mt-cross}
 \E(RPW)=mq\,\mathcal O_{\mathcal G}(P).
\end{equation}
For the last term, its $(g,g)$ block is
$mq(P_{gg}-P_{gg}^2)$ by $P^2=P$.  When $g\ne h$, only intermediate groups
different from both $g$ and $h$ survive, giving
\[
 \E(WPRPW)_{gh}
 =mq^2\{P-D_PP-PD_P\}_{gh}.
\]
Combining the diagonal and off-block cases yields
$\E(WPRPW)=mq\mathcal C_{q,\mathcal G}(P)$.  Averaging
\eqref{eq:supp-mt-expand} over the sample proves
\eqref{eq:masked-target-general-reduction}.
\end{proof}

{Unlike the full-output loss, the result contains nonlinear functions of
$P$ and is not generally a spectral estimator.  Nevertheless, it can be
compared locally with diagonal attenuation by considering the same
retained--omitted direction pairs used in \cref{sec:balanced-risk}.}

\subsection{Masked-target random slopes at the PCA target}
\label{sec:supp-masked-target-slopes}

{Fix a retained population direction $u_j$, $j\leq d$, and an omitted
direction $u_k$, $k>d$.  The rotation that replaces $u_j$ by a direction in
their span is}
\begin{equation}\label{eq:one-direction-path}
 u_{jk}(s)=\frac{u_j+s u_k}{\sqrt{1+s^2}},
 \qquad
 P_{jk}(s)=P_*-u_ju_j^\trans+u_{jk}(s)u_{jk}(s)^\trans.
\end{equation}
{Here $s=\tan(\theta)$ parameterizes the rotation angle.  Its exact clean
reconstruction cost is}
\begin{equation}\label{eq:one-direction-clean-risk}
 \Rex\{P_{jk}(s)\}
 =\frac{(\lambda_j-\lambda_k)s^2}{1+s^2}.
\end{equation}

{We specialize to coordinate masks with $0<q<1$, so $mq>0$, and remove
this positive factor and the constant term in
\eqref{eq:masked-target-general-reduction}.  Thus write}
\begin{equation}\label{eq:masked-target-normalized-loss}
 \Phi_q(P;S)
 =-2\inner{S}{\off(P)}+\inner{S}{\mathcal C_q(P)},
\end{equation}
{where $\mathcal C_q$ is
\eqref{eq:masked-target-general-correction} with singleton coordinate groups.
For a rank-$d$ projector $P$ and a symmetric matrix $S$, define the symmetric
matrix $\Gamma_q(P;S)$ entrywise by}
\begin{equation}\label{eq:masked-target-gradient}
 [\Gamma_q(P;S)]_{ab}
 =\begin{cases}
 \{-2+q(1-P_{aa}-P_{bb})\}S_{ab}, & a\ne b,\\[2pt]
 (1-2P_{aa})S_{aa}
 -2q\displaystyle\sum_{c\ne a}S_{ac}P_{ac}, & a=b.
 \end{cases}
\end{equation}

\begin{proposition}[Masked-target slopes at the PCA target]
\label{prop:masked-target-slopes}
{Let $P_*$ be the leading rank-$d$ projector of $\bSigma$, and let
$P_{jk}(s)$ be the rotation path in \eqref{eq:one-direction-path}.  For every
$j\leq d<k$ and every symmetric $S$,}
\begin{equation}\label{eq:masked-target-directional-slope}
 \left.\frac{d}{ds}\Phi_q\{P_{jk}(s);S\}\right|_{s=0}
 =2u_k^\trans\Gamma_q(P_*;S)u_j.
\end{equation}
{Consequently, $P_*$ is stationary for the population masked-target
objective if and only if}
\begin{equation}\label{eq:masked-target-stationarity}
 P_*^\perp\Gamma_q(P_*;\bSigma)P_*=0.
\end{equation}
{Suppose
$u_k^\trans\Gamma_q(P_*;\bSigma)u_j=0$ for a particular pair $(j,k)$ and
define the population profile curvature and centered single-observation slope
by}
\begin{align}
 \kappa_{jk,q}
 &=\left.\frac{d^2}{ds^2}
   \Phi_q\{P_{jk}(s);\bSigma\}\right|_{s=0},
 \label{eq:masked-target-profile-curvature}\\
 \psi_{jk,q}(x)
 &=2u_k^\trans
   \{\Gamma_q(P_*;xx^\trans)-\Gamma_q(P_*;\bSigma)\}u_j.
 \label{eq:masked-target-profile-score}
\end{align}
{If $\kappa_{jk,q}>0$ and $\widehat s_{jk,q}^{\rm MT}$ is the local
minimizer of the empirical profile converging to zero, then, under
\cref{ass:gaussian-sampling},}
\begin{align}
 \widehat s_{jk,q}^{\rm MT}
 &=-\frac{n^{-1}\sum_{i=1}^n\psi_{jk,q}(\bx_i)}{\kappa_{jk,q}}
   +O_{\mathbb P}(n^{-1}),
 \label{eq:masked-target-profile-minimizer}\\
 \sqrt n\,\widehat s_{jk,q}^{\rm MT}
 &\ \Longrightarrow\
 \mathcal N\!\left(0,
 \frac{\var\{\psi_{jk,q}(\bx)\}}{\kappa_{jk,q}^2}\right),
 \label{eq:masked-target-profile-limit}\\
 n\Rex\{P_{jk}(\widehat s_{jk,q}^{\rm MT})\}
 &\ \Longrightarrow\
 (\lambda_j-\lambda_k)
 \frac{\var\{\psi_{jk,q}(\bx)\}}{\kappa_{jk,q}^2}\,\chi_1^2.
 \label{eq:masked-target-profile-risk}
\end{align}
\end{proposition}

\begin{proof}[Proof of \Cref{prop:masked-target-slopes}]
{Put $\Pi=\diag(P)$ and $\dot\Pi=\diag(H)$.  Differentiating
\eqref{eq:masked-target-normalized-loss} in an arbitrary symmetric direction
$H$ gives}
\begin{align*}
 D\Phi_q(P;S)[H]
 &=-2\inner{S}{\off(H)}
   +\inner{S}{\dot\Pi-\dot\Pi\Pi-\Pi\dot\Pi}\\
 &\quad
   +q\inner{\off(S)}{
   H-\dot\Pi P-\Pi H-H\Pi-P\dot\Pi}.
\end{align*}
{Collecting the coefficient of each entry of $H$ yields
$D\Phi_q(P;S)[H]=\inner{\Gamma_q(P;S)}{H}$ with
$\Gamma_q$ as in \eqref{eq:masked-target-gradient}.  Because
$P_{jk}'(0)=u_ju_k^\trans+u_ku_j^\trans$,
\eqref{eq:masked-target-directional-slope} follows.  Such matrices span the
tangent space of rank-$d$ projectors at $P_*$, proving
\eqref{eq:masked-target-stationarity}.}

{Under stationarity, the empirical derivative at zero is
$n^{-1}\sum_i\psi_{jk,q}(\bx_i)$.  A second-order Taylor expansion of the
profile score, Gaussian moment bounds, and $\kappa_{jk,q}>0$ give
\eqref{eq:masked-target-profile-minimizer}.  {The central limit theorem and
Slutsky's theorem then give \eqref{eq:masked-target-profile-limit}.  Finally,
the exact identity \eqref{eq:one-direction-clean-risk} and the continuous
mapping theorem give \eqref{eq:masked-target-profile-risk}.}}
\end{proof}

{This result exposes the distinction between the two objectives.  Under
balanced marginal variances, diagonal attenuation automatically has zero
population leakage and retains the PCA curvature.  The masked-target objective
still requires the separate stationarity condition
\eqref{eq:masked-target-stationarity} and positive profile curvature.  For the
joint rank-$d$ estimator, the scalar curvatures in
\eqref{eq:masked-target-profile-risk} are replaced by the Hessian of
$\Phi_q$ on the full tangent space; cross-direction Hessian terms need not
vanish.}

{A closed-form example shows that stationarity alone is insufficient.
Let $p=2$, $d=1$,
$u_1=u=(1,1)^\trans/\sqrt2$, $u_2=v=(1,-1)^\trans/\sqrt2$, and
$\bSigma=\lambda_1uu^\trans+\lambda_2vv^\trans$.  Along $P_{12}(s)$, the
single-observation slope is identically zero.  Substituting
$[P_{12}(s)]_{12}=(1-s^2)/[2(1+s^2)]$ into
\eqref{eq:masked-target-normalized-loss} gives}
\begin{equation}\label{eq:masked-target-p2-curvature}
 \kappa_{12,q}=2(\lambda_1-3\lambda_2).
\end{equation}
{Equivalently, the population loss is a convex quadratic in
$t=P_{12}$ with unconstrained minimizer
$t=(\lambda_1-\lambda_2)/(\lambda_1+\lambda_2)$ and feasible range
$[-1/2,1/2]$.  Thus the population PCA direction is the masked-target solution when
$\lambda_1/\lambda_2\geq3$ and is not even a local minimum when
$1<\lambda_1/\lambda_2<3$.  The familiar threshold is therefore a curvature
condition.  In dimensions above two, the population slope need not vanish.
When the stationarity and positive-curvature conditions hold, its empirical
fluctuation produces the usual $n^{-1/2}$ local displacement.}

\subsection{Linear reconstruction beyond the projector constraint}
\label{sec:supp-ordinary-ae}

The main estimator constrains the reconstruction map to be an orthogonal
rank-$d$ projector because the target is PCA.  This section removes that
constraint.  It gives an exact characterization for a symmetric-map
(weight-tied) linear autoencoder and for a fully asymmetric encoder--decoder
pair.  The latter is
the coordinate-mask specialization of the linear full-output model analyzed by
\citet{bisulco2025linearity}.  The calculation also makes precise the
direction-specific gain shrinkage studied for additive-noise autoencoders by
\citet{pretorius2018learning}.

{Because these maps need not be projectors, define their clean
reconstruction risk by}
\begin{equation}\label{eq:clean-map-risk}
 \cR_{\rm map}(A)
 =\E\norm{\bx-A\bx}_2^2
 =\tr\{(I_p-A)^\trans(I_p-A)\bSigma\}.
\end{equation}
{For an orthogonal projector $P$, this agrees with $\cR(P)$ in
\eqref{eq:clean-risk}.}

Write $D_n=\diag(\bS_n)$ and, for $0<q<1$, define
\begin{equation}\label{eq:ae-metric}
  V_{n,q}=q\bS_n+(1-q)D_n.
\end{equation}
If every sample coordinate variance is positive, then $V_{n,q}\succ0$ even
when $\bS_n$ is singular.  {An asymmetric rank-$d$ encoder--decoder has
end-to-end map $A=LR^\trans$, where $L,R\in\RR^{p\times d}$; conversely, every
matrix with rank at most $d$ admits such a factorization.  Weight tying sets
$L=R=B$ and hence restricts $A$ to $BB^\trans\succeq0$.}  For either class,
consider the same full-output masked-input loss as in the main text,
\begin{equation}\label{eq:ordinary-ae-loss}
  \mathcal L_q(A)
  =\frac1n\sum_{i=1}^n
    \E_W\norm{\bx_i-AW\bx_i}_2^2.
\end{equation}

\begin{theorem}[Full-output linear autoencoders]\label{thm:ordinary-linear-ae}
{Assume $0<q<1$ and $V_{n,q}\succ0$.}
\begin{enumerate}[(i)]
\item {\emph{Asymmetric map.} Put
$C_{n,q}=\bS_nV_{n,q}^{-1/2}$.  Every global minimizer over
$\rank(A)\leq d$ has the form}
\begin{equation}\label{eq:untied-ae-solution}
  A_{\rm asym}=C_dV_{n,q}^{-1/2},
  \qquad
  C_d\in\argmin_{\rank(C')\leq d}\norm{C'-C_{n,q}}_{\F}.
\end{equation}
{If the $d$-th and $(d+1)$-th singular values of $C_{n,q}$ are distinct,
this map is unique and its column space is the leading eigenspace of}
\begin{equation}\label{eq:untied-ae-range}
  \bS_nV_{n,q}^{-1}\bS_n.
\end{equation}

\item {\emph{Symmetric map.} Every tied map
$A=BB^\trans\succeq0$ with rank at most $d$ can be written as
$A=U\diag(a_1,\ldots,a_d)U^\trans$, where $U^\trans U=I_d$ and
$a_r\geq0$, allowing $a_r=0$ when the rank is below $d$.  Here $a_r$ is the
reconstruction multiplier along $u_r$:
$Au_r=a_ru_r$.  For fixed $U=(u_1,\ldots,u_d)$, minimizing the loss over
these multipliers gives}
\begin{equation}\label{eq:tied-ae-gains}
  a_r^*(U)=
  \frac{u_r^\trans\bS_nu_r}{u_r^\trans V_{n,q}u_r}.
\end{equation}
{Consequently, the global tied directions solve}
\begin{equation}\label{eq:tied-ae-directions}
  \max_{U^\trans U=I_d}
  \sum_{r=1}^d
  \frac{(u_r^\trans\bS_nu_r)^2}
       {u_r^\trans V_{n,q}u_r}.
\end{equation}
\end{enumerate}
\end{theorem}

{In general, these solutions differ from both ordinary PCA and the
attenuated spectral estimator $P_{\bS_n-\alpha D_n}$.}

\paragraph{Balanced population specialization.}
{Suppose $\diag(\bSigma)=\tau I_p$ and
$\bSigma=\sum_j\lambda_ju_ju_j^\trans$, with
$\lambda_d>\lambda_{d+1}$.  The asymmetric population map is}
\begin{equation}\label{eq:balanced-ae-population-map}
  A_{q,{\rm asym}}^{\rm pop}
  =\sum_{j=1}^d
  \frac{\lambda_j}{q\lambda_j+(1-q)\tau}
  u_ju_j^\trans.
\end{equation}
{Its range is the PCA subspace, but its clean reconstruction risk relative
to the PCA projector is}
\begin{equation}\label{eq:balanced-ae-population-bias}
  \cR_{\rm map}(A_{q,{\rm asym}}^{\rm pop})-\cR(P_*)
  =\sum_{j=1}^d\lambda_j
  \left\{
  \frac{(1-q)(\tau-\lambda_j)}
       {q\lambda_j+(1-q)\tau}
  \right\}^2.
\end{equation}
{For $q<1$, this excess is strictly positive whenever some retained
eigenvalue differs from $\tau$.  Thus the asymmetric autoencoder learns the
correct subspace, while its raw reconstruction map is generally not the
projector $P_*$.}

\begin{proof}[{Proof of \Cref{thm:ordinary-linear-ae} and the balanced
population specialization}]
For an arbitrary $A$, expansion of \eqref{eq:ordinary-ae-loss} gives
\begin{align}
 \mathcal L_q(A)
 &=\tr(\bS_n)-2q\tr(A\bS_n)
   +q\tr(AV_{n,q}A^\trans)\nonumber\\
 &=\tr(\bS_n)-q\tr(\bS_nV_{n,q}^{-1}\bS_n)
   +q\norm{AV_{n,q}^{1/2}-\bS_nV_{n,q}^{-1/2}}_{\F}^2.
 \label{eq:ordinary-ae-square}
\end{align}
The first equality uses
$\E(W\bS_nW)=q^2\bS_n+q(1-q)D_n=qV_{n,q}$.
For the second, expand the squared Frobenius norm and use cyclicity of trace:
its cross term is
$-2q\tr\{AV_{n,q}^{1/2}V_{n,q}^{-1/2}\bS_n\}
=-2q\tr(A\bS_n)$.

Right multiplication by the invertible $V_{n,q}^{1/2}$ preserves rank.  Thus
minimizing \eqref{eq:ordinary-ae-square} over $\rank(A)\leq d$ is exactly the
best rank-at-most-$d$ approximation of
$C_{n,q}=\bS_nV_{n,q}^{-1/2}$.  The Eckart--Young theorem gives
\eqref{eq:untied-ae-solution}.  The left singular vectors of $C_{n,q}$ are the
eigenvectors of
$C_{n,q}C_{n,q}^\trans=\bS_nV_{n,q}^{-1}\bS_n$, proving
\eqref{eq:untied-ae-range}.

{With a singular-value gap, \eqref{eq:untied-ae-solution} yields}
\begin{equation}\label{eq:untied-ae-projector-form}
  A_{\rm asym}
  =P_{\bS_nV_{n,q}^{-1}\bS_n}\bS_nV_{n,q}^{-1}.
\end{equation}

For a tied map, substitute
$A=U\diag(a_1,\ldots,a_d)U^\trans$ into the first line of
\eqref{eq:ordinary-ae-square}.  The part depending on $A$, divided by $q$, is
\begin{equation}\label{eq:tied-ae-fixed-u-loss}
 -2\sum_{r=1}^d a_r u_r^\trans\bS_nu_r
 +\sum_{r=1}^d a_r^2u_r^\trans V_{n,q}u_r.
\end{equation}
There are no cross terms in the final sum: the trace selects only the diagonal
entries after the two diagonal coefficient matrices are applied.  Each scalar quadratic in
\eqref{eq:tied-ae-fixed-u-loss} is strictly convex and has minimizer
\eqref{eq:tied-ae-gains}.  Substitution subtracts
$(u_r^\trans\bS_nu_r)^2/(u_r^\trans V_{n,q}u_r)$ from the loss, yielding
\eqref{eq:tied-ae-directions}.

{Under population balance,
$V_q=q\bSigma+(1-q)\tau I_p$ shares every eigenvector with $\bSigma$.
For an eigenvalue $\lambda>0$, the squared singular value that orders the
asymmetric solution is}
\begin{equation}\label{eq:balanced-ae-ordering}
  f_q(\lambda)=\frac{\lambda^2}{q\lambda+(1-q)\tau}.
\end{equation}
{Its derivative is}
\[
 f_q'(\lambda)=
 \frac{\lambda\{q\lambda+2(1-q)\tau\}}
      {\{q\lambda+(1-q)\tau\}^2}>0,
\]
{so the asymmetric solution selects $u_1,\ldots,u_d$.
Equation~\eqref{eq:untied-ae-solution} gives its multiplier
$\lambda_j/\{q\lambda_j+(1-q)\tau\}$, proving
\eqref{eq:balanced-ae-population-map}.  Finally,}
\[
 \cR_{\rm map}(A)-\cR(P_*)
 =\sum_{j=1}^d\lambda_j(1-a_j)^2
\]
{for any map diagonal in the population eigenbasis that acts as zero on
the orthogonal complement of the leading subspace.  Substituting the displayed
multiplier proves \eqref{eq:balanced-ae-population-bias} and its strictness
statement.}

\end{proof}

\subsection{Additional related context}\label{sec:supp-related}

{Input corruptions can induce coordinate-dependent regularization
\citep{wager2013dropout}.  For linear denoising autoencoders,
\citet{pretorius2018learning} study direction-specific reconstruction gains
under additive isotropic noise, while \citet{kong2023understanding} relate the
mask ratio and patch size of a hidden-target objective to identifiable levels
of a hierarchical latent model.  Missing-data PCA instead observes physically
reduced vectors and estimates an unobserved signal covariance
\citep{dobriban2016pca}.  These settings motivate nearby uses of corruption or
diagonal operators, but they differ from the complete-vector PCA target and
clean reconstruction risk studied here.  The exact objective and feasible-set
boundaries needed for that distinction are given in
\cref{sec:supp-masked-target-slopes,sec:supp-ordinary-ae}.}

\section{Additional numerical and real-data studies}
\label{sec:supp-numerics}

{The main text reports Gaussian experiments for preserved and moving
population targets.  This section adds a non-Gaussian stress test and gives the
complete protocol and cellwise results for the fixed fitting-sample-budget
real-data study.}

\subsection{Robustness beyond Gaussian observations}

{The non-Gaussian audit uses the same balanced $p=16,d=3$ covariance as
the general Gaussian experiment.  Before multiplication by the covariance
square root, the latent coordinates are either Gaussian, standardized Student
$t_{10}$, or standardized centered exponential.  All three constructions have
covariance $\bSigma$ and finite eighth moments.  A separate pilot of 150,000
observations per distribution estimates the constants in
\eqref{eq:moment-risk-quadratic}; each sample size
$n\in\{64,256,1024\}$ then uses 5,000 fresh common-sample replicates, with all
attenuation strengths evaluated on the same observations.  At $n=1024$,
the largest absolute difference between the simulated scaled risk and the
moment prediction across all displayed distributions and attenuation levels
is $1.02$.  The sign and curvature predicted by
\cref{thm:moment-extension} therefore persist under both heavy-tailed and
skewed observations in this controlled design.}

{\Cref{fig:non-gaussian-moment} displays this agreement between the
simulated curves and the fourth-moment prediction across all three latent
distributions.}

\begin{figure}[t!]
  \centering
  \includegraphics[width=0.92\linewidth]
  {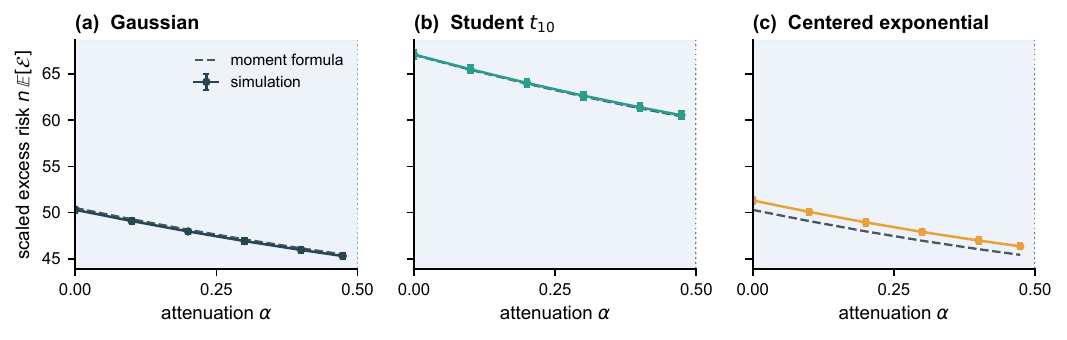}
  \caption{{The moment risk formula beyond Gaussian observations.  Solid lines
  and intervals are Monte Carlo estimates at $n=1024$; dashed lines are the
  corresponding fourth-moment predictions.  The covariance and target rank
  are common across panels, and shading marks the mask-feasible interval.}}
  \label{fig:non-gaussian-moment}
\end{figure}

\subsection{Fixed-budget real-data comparisons}
\label{sec:supp-real-data}
\label{sec:supp-masked-target-comparison}

\paragraph{Implementation details.}
{The data, ranks, patch construction, block definitions, and fixed
fitting-sample-budget protocol are given in \cref{sec:experiments}; here we
record the additional details needed to reproduce the comparison.  Every
covariance matrix is centered within its fitting fold, and no coordinate
rescaling is applied except within the correlation-PCA baseline.  When a
validation curve has multiple minimizers, we choose the smallest attenuation
strength or mask rate.  For Optdigits, each of the ten patch-location seeds
{contributes a mean over 100 training subsamples}, and intervals use the standard
error of these ten seed-level means.  For Phoneme and HAR, intervals use the
standard error across 100 training subsamples, conditional on the observed
data and official split.  Within every subsample, all methods use the same
observations and validation folds.}

{HeteroPCA starts from the hollowed sample covariance.  At each iteration
it replaces the diagonal by that of the best rank-$d$ truncated-SVD
approximation while retaining the original sample off-diagonal; it stops when
the relative change of the imputed diagonal is at most $10^{-7}$ or after 200
iterations \citep{zhang2022heteroskedastic}.  Of its {2,400} fits, 133 reach the
iteration cap.  Each masked-target projector is approximately
optimized on the Grassmann manifold from both the PCA solution and the
corresponding full-output block-attenuation solution.  Optimization stops when
the relative tangent-gradient norm is at most $10^{-7}$ or after 800
iterations.  Among the {1,163} coordinate-masked fits that do not select the PCA
fallback, 86 reach this cap; none of the {175} nonfallback block-masked fits does.
We therefore report masked-target training as a two-start optimization rather
than a certified global optimum.}

\begin{table}[t!]
\centering
\small
\setlength{\tabcolsep}{5pt}
\caption{{Fixed fitting-sample-budget results for diagonal attenuation.  Entries are the mean percentage of PCA test excess reconstruction risk removed, with 95\% intervals in brackets.  Optdigits intervals treat the ten patch-location seeds as independent clusters, with each seed-level mean averaging 100 training subsamples; the other intervals use 100 training subsamples.}}
\label{tab:attenuation-real-data}
{
\begin{tabular}{llrr}
\toprule
Data & $N$ & Mask-derived DA & Direct DA \\
\midrule
Optdigits & $64$ & $4.32$ $[4.08,4.57]$ & $6.74$ $[6.30,7.17]$ \\
Optdigits & $128$ & $4.18$ $[3.85,4.51]$ & $6.40$ $[5.80,7.00]$ \\
\addlinespace[2pt]
Phoneme & $64$ & $0.84$ $[0.65,1.03]$ & $1.09$ $[0.78,1.40]$ \\
Phoneme & $128$ & $0.72$ $[0.51,0.93]$ & $0.70$ $[0.36,1.03]$ \\
\addlinespace[2pt]
HAR & $64$ & $0.65$ $[0.45,0.85]$ & $0.76$ $[0.48,1.05]$ \\
HAR & $128$ & $0.93$ $[0.78,1.07]$ & $1.20$ $[0.97,1.44]$ \\
\bottomrule
\end{tabular}}
\end{table}

\begin{table}[t!]
\centering
\scriptsize
\setlength{\tabcolsep}{2.5pt}
\caption{{Comparison methods across the three real-data studies.  Each entry is the mean percentage of PCA test excess reconstruction risk removed, with its 95\% interval in brackets.  Optdigits intervals cluster over ten patch-location seeds, each averaging 100 training subsamples; the remaining intervals use 100 training subsamples.  MT denotes masked-target training.}}
\label{tab:neighboring-real-data}
{
\begin{tabular}{llrrrr}
\toprule
Data & $N$ & Correlation PCA & HeteroPCA & Coordinate MT & Block MT \\
\midrule
Optdigits & $64$ & $5.84$ $[5.34,6.33]$ & $-25.79$ $[-27.60,-23.99]$ & $-1.63$ $[-2.01,-1.25]$ & $-3.58$ $[-4.55,-2.61]$ \\
Optdigits & $128$ & $5.41$ $[4.99,5.82]$ & $-37.36$ $[-41.93,-32.79]$ & $-1.24$ $[-1.80,-0.68]$ & $-1.27$ $[-1.98,-0.56]$ \\
\addlinespace[2pt]
Phoneme & $64$ & $-0.14$ $[-0.70,0.42]$ & $0.04$ $[-0.53,0.60]$ & $0.93$ $[0.57,1.28]$ & $0.00$ $[0.00,0.00]$ \\
Phoneme & $128$ & $-0.52$ $[-1.06,0.01]$ & $0.53$ $[0.42,0.64]$ & $0.28$ $[-0.10,0.66]$ & $0.00$ $[0.00,0.00]$ \\
\addlinespace[2pt]
HAR & $64$ & $0.36$ $[0.06,0.66]$ & $-13.94$ $[-15.46,-12.43]$ & $0.56$ $[0.28,0.83]$ & $0.00$ $[0.00,0.00]$ \\
HAR & $128$ & $1.01$ $[0.74,1.28]$ & $-4.87$ $[-6.23,-3.50]$ & $0.93$ $[0.64,1.23]$ & $0.00$ $[0.00,0.00]$ \\
\bottomrule
\end{tabular}}
\end{table}

\begin{table}[t!]
\centering
\small
\setlength{\tabcolsep}{3.5pt}
\caption{{Mean cross-validation-selected strengths under fitting-sample budget $N$.  DA denotes diagonal attenuation and MT denotes masked-target training.  Parentheses give the percentage of fits selecting the explicit PCA fallback.  DA columns report $\alpha$; MT columns report the mask rate $m$.}}
\label{tab:real-data-selected-strengths}
{
\begin{tabular}{@{}llcccc@{}}
\toprule
Data & $N$ & Mask DA $\alpha$ & Direct DA $\alpha$ & Coordinate MT $m$ & Block MT $m$ \\
\midrule
Optdigits & $64$ & $0.474$ $(1.7\%)$ & $0.927$ $(1.7\%)$ & $0.432$ $(51.8\%)$ & $0.121$ $(86.6\%)$ \\
Optdigits & $128$ & $0.479$ $(1.1\%)$ & $0.931$ $(1.1\%)$ & $0.323$ $(64.1\%)$ & $0.037$ $(95.9\%)$ \\
\addlinespace[2pt]
Phoneme & $64$ & $0.432$ $(7.0\%)$ & $0.800$ $(7.0\%)$ & $0.589$ $(18.0\%)$ & $0.000$ $(100.0\%)$ \\
Phoneme & $128$ & $0.434$ $(7.0\%)$ & $0.764$ $(7.0\%)$ & $0.479$ $(15.0\%)$ & $0.000$ $(100.0\%)$ \\
\addlinespace[2pt]
HAR & $64$ & $0.351$ $(22.0\%)$ & $0.615$ $(18.0\%)$ & $0.370$ $(27.0\%)$ & $0.000$ $(100.0\%)$ \\
HAR & $128$ & $0.418$ $(9.0\%)$ & $0.756$ $(8.0\%)$ & $0.563$ $(18.0\%)$ & $0.000$ $(100.0\%)$ \\
\bottomrule
\end{tabular}}
\end{table}

{\Cref{tab:attenuation-real-data} gives the complete attenuation
results.  Both paths have positive mean gains with intervals above zero in all
six cells.  \Cref{tab:neighboring-real-data} reports the four comparison
methods on all three data sets.  Correlation PCA is competitive on Optdigits
and HAR but does not improve Phoneme.  Coordinate masked-target training is
positive in three ordered-signal cells but not on the image patches, while the
block-masked family usually selects its explicit PCA fallback.  Thus none of
these alternatives reproduces the uniform improvement of diagonal
attenuation.}

{\Cref{tab:real-data-selected-strengths} shows how the validation rule
uses the path.  Across the three data sets, mean selected strengths are
$0.351$--$0.479$ on the mask-derived path and {$0.615$--$0.931$} on the full
direct path.  The direct search selects strengths beyond $1/2$ on
{$63\%$--$95\%$} of fits, while retaining PCA as a candidate in every search.}

\end{document}